%% file: main.tex
\documentclass[11pt]{article}

\usepackage{fancyhdr}
\usepackage[colorlinks,citecolor=blue,urlcolor=blue,linkcolor=blue,bookmarks=false]{hyperref}
\usepackage{amsfonts,epsfig,graphicx}
\usepackage{afterpage}
\usepackage{mathrsfs}
\usepackage{amsmath,amssymb, amsthm, dsfont}
\usepackage{fullpage}
\usepackage{bbm}
\usepackage{bm}
\usepackage{enumerate}
\usepackage[T1]{fontenc}
\usepackage{algorithm}
\usepackage{algorithmic}
\usepackage{subfigure}
\usepackage[authoryear, round, colon]{natbib}
\usepackage{multirow}

\input{nonlinear_marcos}

\begin{document}

\title{\Huge Optimal linear estimation under unknown nonlinear transform}

\author
{
Xinyang Yi\thanks{Department of Electrical and Computer Engineering, The University of Texas at Austin, Austin, TX 78712, USA; e-mail: {\tt
yixy@utexas.edu}.} 
\qquad Zhaoran Wang\thanks{Department of Operations Research and Financial Engineering, Princeton University, Princeton, NJ 08544, USA; e-mail: {\tt
zhaoran@princeton.edu}.} 
\qquad Constantine Caramanis\thanks{Department of Electrical and Computer Engineering, The University of Texas at Austin, Austin, TX 78712, USA; e-mail: {\tt
constantine@utexas.edu}.}
\qquad Han Liu\thanks{Department of Operations Research and Financial Engineering, Princeton University, Princeton, NJ 08544, USA; e-mail: {\tt
hanliu@princeton.edu}.}
}
\date{}

\maketitle

\begin{abstract}
\noindent
Linear regression studies the problem of estimating a model parameter $\bbeta^* \!\in\! \RR^p$, from $n$ observations $\{(y_i,\xb_i)\}_{i=1}^n$ from linear model $y_i = \langle \xb_i,\bbeta^* \rangle + \epsilon_i$. We consider a significant generalization in which the relationship between $\langle \mathbf{x}_i,\bbeta^* \rangle$ and $y_i$ is noisy, quantized to a single bit, potentially nonlinear, noninvertible, as well as unknown. This model is known as the single-index model in statistics, and, among other things, it represents a significant generalization of one-bit compressed sensing. We propose a novel spectral-based estimation procedure and show that we can recover $\bbeta^*$ in settings (i.e., classes of link function $f$) where previous algorithms fail. In general, our algorithm requires only very mild restrictions on the (unknown) functional relationship between $y_i$ and $\langle \mathbf{x}_i,\bbeta^* \rangle$. We also consider the high dimensional setting where $\bbeta^*$ is sparse ,and introduce a two-stage nonconvex framework that addresses estimation challenges in high dimensional regimes where $p \gg n$. For a broad class of link functions between $\langle \mathbf{x}_i,\bbeta^* \rangle$ and $y_i$, we establish minimax lower bounds that demonstrate the optimality of our estimators in both the classical and high dimensional regimes.
\end{abstract}

\newpage
\input{intro}

\input{main_result}

\input{experiment}

\input{proof}

\clearpage
\newpage
\bibliographystyle{ims}
\bibliography{main}
\clearpage
\end{document}

%% file: nonlinear_marcos.tex
\def\B{{\mathbb B}}
\def\R{{\mathbb R}}

\long\def\comment#1{}

\def\dr{\displaystyle \rm}

\newcommand{\bel}{\begin{eqnarray}\label}
\newcommand{\eel}{\end{eqnarray}}
\newcommand{\bes}{\begin{eqnarray*}}
	\newcommand{\ees}{\end{eqnarray*}}

\def\R{{\real}}

\newcommand{\la}{\langle}
\newcommand{\ra}{\rangle}
\newcommand{\cIs}{\cI_{\hat{s}}}
\newcommand{\Tr}{\operatorname{Tr}}
\newcommand{\abs}[1]{\left|#1\right|}

\def\sep{\;\big|\;}

\def\##1\#{\begin{align}#1\end{align}}
\def\$#1\${\begin{align*}#1\end{align*}}

\let\hat\widehat

\newcommand{\xb}{\mathbf{x}}

\newcommand{\bw}{\bm{w}}

\newcommand{\Ib}{\mathbf{I}}

\newcommand{\Mb}{\mathbf{M}}

\newcommand{\Qb}{\mathbf{Q}}

\newcommand{\Ub}{\mathbf{U}}

\newcommand{\Xb}{\mathbf{X}}

\newcommand{\bX}{\bm{X}}

\newcommand{\cB}{\mathcal{B}}

\newcommand{\cF}{\mathcal{F}}

\newcommand{\cI}{\mathcal{I}}

\newcommand{\cN}{\mathcal{N}}

\newcommand{\cQ}{\mathcal{Q}}

\newcommand{\cX}{\mathcal{X}}

\newcommand{\EE}{\mathbb{E}}

\newcommand{\PP}{\mathbb{P}}

\newcommand{\RR}{\mathbb{R}}
\newcommand{\SSS}{\mathbb{S}}

\newcommand{\bbeta}{\bm{\beta}}

\newcommand{\btheta}{\bm{\theta}}

\newcommand{\bPi}{\bm{\Pi}}

\newcommand{\argmin}{\mathop{\mathrm{argmin}}}

\newcommand{\sign}{\mathop{\mathrm{sign}}}

\newcommand{\supp}{\mathop{\mathrm{supp}}}

\DeclareMathOperator*{\ind}{\mathds{1}}  

\newcommand*{\zero}{{\bm 0}}

\global\long\def\B{\bm{B}}

\global\long\def\R{\bm{R}}

\global\long\def\x{\bm{x}}

\global\long\def\e{\bm{e}}

\newcommand{\overbar}[1]{\mkern 1.5mu\overline{\mkern-1.5mu#1\mkern-1.5mu}\mkern 1.5mu}
\newcommand{\ud}{{\mathrm{d}}}

\newtheoremstyle{mytheoremstyle} 
    {\topsep}                    
    {\topsep}                    
    {\normalfont}                   
    {}                           
    {\bfseries}                   
    {.}                          
    {.5em}                       
    {}  

\theoremstyle{mytheoremstyle}

\ifx\BlackBox\undefined
\newcommand{\BlackBox}{\rule{1.5ex}{1.5ex}}  
\fi

\ifx\QED\undefined
\def\QED{~\rule[-1pt]{5pt}{5pt}\par\medskip}
\fi

\ifx\proof\undefined
\newenvironment{proof}{\par\noindent{\bf Proof\ }}{\hfill\BlackBox\\[2mm]}
\fi

\ifx\theorem\undefined
\newtheorem{theorem}{Theorem}
\fi
\ifx\example\undefined
\newtheorem{example}[theorem]{Example}
\fi
\ifx\property\undefined
\newtheorem{property}{Property}
\fi
\ifx\lemma\undefined
\newtheorem{lemma}[theorem]{Lemma}
\fi
\ifx\proposition\undefined
\newtheorem{proposition}[theorem]{Proposition}
\fi
\ifx\remark\undefined
\newtheorem{remark}[theorem]{Remark}
\fi
\ifx\corollary\undefined
\newtheorem{corollary}[theorem]{Corollary}
\fi
\ifx\definition\undefined
\newtheorem{definition}[theorem]{Definition}
\fi
\ifx\conjecture\undefined
\newtheorem{conjecture}[theorem]{Conjecture}
\fi
\ifx\fact\undefined
\newtheorem{fact}[theorem]{Fact}
\fi
\ifx\claim\undefined
\newtheorem{claim}[theorem]{Claim}
\fi
\ifx\assumption\undefined
\newtheorem{assumption}[theorem]{Assumption}
\fi
\numberwithin{equation}{section}
\numberwithin{theorem}{section}

%% file: intro.tex

\section{Introduction}
We consider a generalization of the one-bit quantized regression problem, where we seek to recover the regression coefficient  $\bbeta^* \!\in\! \RR^p$ from one-bit measurements. Specifically, suppose that $\bX$ is a random vector in $\RR^p$ and $Y$ is a binary random variable taking values in $\{-1,1 \}$. We assume the conditional distribution of $Y$ given $\bX$ takes the form
\#
\label{model}
\PP(Y = 1 | \bX = \x) = \frac{1}{2}f(\langle \x, \bbeta^* \rangle) + \frac{1}{2},
\#
where $f: \mathbb{R} \rightarrow [-1, 1]$ is called the \emph{link function}. We aim to estimate $\bbeta^*$ from $n$ i.i.d. observations $\{(y_i,\xb_i)\}_{i=1}^{n}$ of the pair $(Y,\bX)$. In particular, we assume the link function $f$ is unknown. Without any loss of generality, we take $\bbeta^*$ to be on the unit sphere $\SSS^{p-1}$ since its magnitude can always be incorporated into the link function $f$. 


The model in \eqref{model} is simple but general. Under specific choices of the link function $f$, \eqref{model} immediately leads to many practical models in machine learning and signal processing, including logistic regression and one-bit compressed sensing. In the settings where the link function is assumed to be known, a popular estimation procedure is to calculate an estimator that minimizes a certain loss function. However, for particular link functions, this approach involves minimizing a nonconvex objective function for which the global minimizer is in general intractable to obtain. Furthermore, it is difficult or even impossible to know the link function in practice, and a poor choice of link function may result in inaccurate parameter estimation and high prediction error. We take a more general approach, and in particular, target the setting where $f$ is unknown. We propose an algorithm that can estimate the parameter $\bbeta^*$ in the absence of prior knowledge on the link function $f$. As our results make precise, our algorithm succeeds as long as the function $f$ satisfies a single moment condition. As we demonstrate, this moment condition is only a mild restriction on $f$. In particular, our methods and theory are widely applicable even to the settings where $f$ is non-smooth, e.g., $f(z) = \sign(z)$, or noninvertible, e.g., $f(z) = \sin(z)$. 

In particular, as we show in Section \ref{sec:example_models}, our restrictions on $f$ are sufficiently flexible so that our results provide a unified framework that encompasses a broad range of problems, including logistic regression, one-bit compressed sensing, one-bit phase retrieval as well as their robust extensions. We use these important examples to illustrate our results, and discuss them at several points throughout the paper.

{\bf Main contributions.} The key conceptual contribution of this work is a novel use of the method of moments. Rather than considering moments of the covariate, $\bX$, and the response variable, $Y$, we look at moments of differences of covariates, and differences of response variables. Such a simple yet critical observation enables everything that follows. In particular, it leads to our spectral-based procedure, which provides an effective and general solution for the suite of problems mentioned above. In the low dimensional (or what we refer to as the classical) setting, our algorithm is simple: a spectral decomposition of the moment matrix mentioned above. In the high dimensional setting, when the number of samples, $n$, is far outnumbered by the dimensionality, $p$, important when $\bbeta^*$ is sparse, we use a two-stage nonconvex optimization algorithm to perform the high dimensional estimation.

We simultaneously establish the statistical and computational rates of convergence of the proposed spectral algorithm as well as its consequences for the aforementioned estimation problems. We consider both the low dimensional setting where the number of samples exceeds the dimension (we refer to this as the ``classical'' setting) and the high dimensional setting where the dimensionality may (greatly) exceed the number of samples. In both these settings, our proposed algorithm achieves the same statistical rate of convergence as that of linear regression applied on data generated by the linear model without quantization. Second, we provide minimax lower bounds for the statistical rate of convergence, and thereby establish the optimality of our procedure within a broad model class. In the low dimensional setting, our results obtain the optimal rate with the optimal sample complexity. In the high dimensional setting, our algorithm requires estimating a sparse eigenvector, and thus our sample complexity coincides with what is believed to be the best achievable via polynomial time methods (\cite{berthet2013lowerSparsePCA}); the error rate itself, however, is information-theoretically optimal. We discuss this further in Section \ref{ssec:highdim}.

\vskip4pt
 {\bf Related works.}
Our model in \eqref{model} is close to the \emph{single-index model} (SIM) in statistics. In the SIM, we assume that the response-covariate pair $(Y,\bX)$ is determined by 
\#\label{eq:w01}
Y = f(\langle \bX, \bbeta^* \rangle) + W
\#
with unknown link function $f$ and noise $W$. Our setting is a special case of this, as we restrict $Y$ to be a binary random variable. The single index model is a classical topic, and thus there is extensive literature -- too much to exhaustively review it. We therefore outline the pieces of work most relevant to our setting and our results. For estimating $\bbeta^*$ in \eqref{eq:w01}, a feasible approach is $M$-estimation \citep{hardle1993optimal,delecroix2000optimal, delecroix2006semiparametric}, in which the unknown link function $f$ is jointly estimated using nonparametric estimators. Although these $M$-estimators have been shown to be consistent, they are not computationally efficient since they involve solving a nonconvex optimization problem. Another approach to estimate $\bbeta^*$ is named the \emph{average derivative estimator} (ADE; 
\citet{stoker1986consistent}). Further improvements of ADE are considered in \citet{powell1989semiparametric} and \citet{hristache2001direct}. ADE and its related methods require that the link function $f$ is at least differentiable, and thus excludes important models such as one-bit compressed sensing with $f(z) = \sign(z)$. Beyond estimating $\bbeta^*$, \citet{kalai2009isotron} and \citet{kakade2011efficient} focus on iteratively estimating a function $f$ and vector $\bbeta$ that are good for prediction, and they attempt to control the generalization error. Their algorithms are based on isotonic regression, and are therefore only applicable when the link function is monotonic and satisfies Lipschitz constraints. The work discussed above focuses on the low dimensional setting where $p \ll n$.

Another related line of works is \emph{sufficient dimension reduction}, where the goal is to find a subspace $\Ub$ of the input space such that the response $Y$ only depends on the projection $\Ub^{\top}\bX$. Single-index model and our problem can be regarded as special cases of this problem as we are primarily in interested in recovering a one-dimensional subspace. Most works on this problem are based on spectral methods including sliced inverse regression (SIR; \citet{li1991sliced}), sliced average variance estimation (SAVE; \cite{cook1999dimension}) and principal hessian directions (PHD; \citet{li1992principal,cook1998principal}). The key idea behind these algorithms is to construct certain empirical moments whose population level structures reveal the underlying true subspace. Our moment estimator is partially inspired by this idea. We highlight two differences compared to these existing works. First, our spectral method is based on computing covariance matrix of response weighted sample differences that is not considered in previous works. This special design allows us to deal with both odd and even link functions under mild conditions while SIR and PHD are both limited to only one of the two cases \footnote{In our setting, SIR corresponds to approximating $\mathbb{E}(Y\bX)$ that is $0$ for even link functions; PHD corresponds to approximating $\mathbb{E}(Y\bX\bX^{\top})$ that is $0$ for odd link functions.} and SAVE is more statistically inefficient than ours. Second, all the aforementioned works focus on asymptotic analysis while the performances (e.g., statistical rate) are much less understood under finite samples or even in high dimensional regime. However, dealing with high dimensionality with optimal statistical rate is one of our main contributions.

In the high dimensional regime with $p \gg n$ and $\bbeta^*$ has some structure (for us this means sparsity), we note there exists some recent progress \citep{alquier2013sparse} on estimating $f$ via PAC Bayesian methods. In the special case when $f$ is linear function, sparse linear regression has attracted extensive study over the years (see the book \citet{buhlmann2011statistics} and references therein). The recent work by \citet{Plan2014} is closest to our setting. They consider the setting of normal covariates, $\bX \sim \cN(\zero, \Ib_p)$, and they propose a marginal regression estimator for estimating $\bbeta^*$, that, like our approach, requires no prior knowledge about $f$. Their proposed algorithm relies on the assumption that $\EE_{z \sim \cN(0,1)}\big[z f(z)\big] \ne 0$, and hence cannot work for link functions that are even. As we describe below, our algorithm is based on a novel moment-based estimator, and avoids requiring such a condition, thus allowing us to handle even link functions under a very mild moment restriction, which we describe in detail below. Generally, the work in \citet{Plan2014} requires different conditions, and thus beyond the discussion above, is not directly comparable to the work here. In cases where both approaches apply, the results are minimax optimal.

\section{Example models}
\label{sec:example_models}
In this section, we discuss several popular (and important) models in machine learning and signal processing that fall into our framework \eqref{model} under specific link functions. Variants of these models have been studied extensively in the recent literature. These examples trace through the paper, and we use them to illustrate the details of our algorithms and results.

\subsection{Logistic regression}
Given the response-covariate pair $(Y,\bX) \in \{-1,1\}\times\mathbb{R}^p$ and model parameter $\bbeta^* \in \R^p$, for logistic regression we assume
\begin{equation}
\label{noisyLogit}
\PP(Y = 1 | \bX = \x) = \frac{1}{1+\exp{(-\langle \x, \bbeta^* \rangle -\zeta)}},
\end{equation}
where $\zeta$ is the intercept. Compared with our general model \eqref{model},  we have
\$
f(z) = \frac{\exp{(z+\zeta)}-1}{\exp{(z+\zeta)}+1}.
\$
One robust variant of logistic regression is called \emph{flipped logistic regression}, where we assume that the labels $Y$ generated from (\ref{noisyLogit}) are flipped with probability $p_{\rm e}$, i.e., 
\begin{equation}
\label{noisyLogit1}
\PP(Y = 1 | \bX = \x) = \frac{1 - p_{\rm e}}{1+\exp{(-\langle \x, \bbeta^* \rangle -\zeta)}} +  \frac{p_{\rm e}}{1+\exp{(\langle \x, \bbeta^* \rangle +\zeta)}}. 
\end{equation} 
This reduces to the standard logistic regression model when $p_{\rm e} = 0$.  For flipped logistic regression, the link function $f$ can be written as 
\begin{equation}
\label{model: logitRegression}
f(z) = \frac{\exp{(z+\zeta)}-1}{\exp{(z+\zeta)}+1} + 2p_{\rm e}\cdot\frac{1-\exp{(z+\zeta)}}{1+\exp{(z+\zeta)}}.
\end{equation}
Flipped logistic regression has been studied by \citet{natarajan2013learning} and \citet{tibshirani2013robust}. In both papers, estimating $\bbeta^*$ is based on minimizing some surrogate loss function involving a certain tuning parameter connected to $p_{\rm e}$. However, $p_{\rm e}$ is unknown in practice. In contrast to their approaches, our method does not hinge on the unknown parameter $p_{\rm e}$. In fact, our approach has the same formulation for both standard and flipped logistic regression, and thus unifies the two models. 

\subsection{One-bit compressed sensing}
 One-bit compressed sensing (e.g., \citet{plan2013one, plan2013robust, netrapalli2013one} ) aims at recovering sparse signals from quantized linear measurements. In detail, we define 
 \#\label{eq:w9182}
 \mathbb{B}_0(s, p) = \{\bbeta \in \mathbb{R}^p: |\supp(\bbeta)| \leq s \}
 \#
  as the set of sparse vectors in $\mathbb{R}^p$ with at most $s$ nonzero elements. We assume $(Y,\bX) \in \{-1,1\}\times\RR^p$ satisfies 
\begin{equation}
\label{one-bit CS}
Y = \sign(\langle \bX, \bbeta^*\rangle),
\end{equation}
where $\bbeta^* \in \mathbb{B}_0(s, p)$.  In this paper, we also consider its robust version with noise $\epsilon$, i.e.,
\begin{equation}
\label{robust one-bit CS}
Y = \sign(\langle \bX, \bbeta^*\rangle + \epsilon).
\end{equation}
Under our framework, the model in (\ref{one-bit CS}) corresponds to the link function $f(z) = \sign(z)$. Assuming $\epsilon \sim \cN(0,\sigma^2)$ in \eqref{robust one-bit CS}, the model in (\ref{robust one-bit CS}) corresponds to the link function
\begin{equation}
\label{model:robust one-bit CS}
f(z) = 2\int_{0}^{\infty} \frac{1}{\sqrt{2\pi}\sigma}e^{-{(u - z)^2}/{2\sigma^2}}\ud u- 1.
\end{equation}
It is worth pointing out that \eqref{robust one-bit CS} also corresponds to the probit regression model without the sparse constraint on $\bbeta^*$. Throughout the paper, we do not distinguish between the two model names. Model \eqref{robust one-bit CS} is referred to as one-bit compressed sensing even in the case where $\bbeta^*$ is not sparse.
\subsection{One-bit phase retrieval}
The goal of phase retrieval (e.g., \citet{candes2013phase,chen2014convex, candes2014phase}) is to recover signals based on linear measurements with phase information erased, i.e., pair $(Y,\bX) \in \RR\times\RR^p$ is determined by equation 
\$
Y = |\langle\bX, \bbeta^*\rangle|.
\$
Analogous to one-bit compressed sensing, we consider a new model named \emph{one-bit phase retrieval} where the linear measurement with phase information erased is quantized to one bit. In detail, pair $(Y,\bX) \in \{-1,1\}\times\RR^p$ is linked through
\#\label{eq:w02}
Y = \sign(|\langle\bX, \bbeta^*\rangle| - \theta). 
\#
where $\theta$ is the quantization threshold. Compared with one-bit compressed sensing, this problem is more difficult because $Y$ only depends on $\bbeta^*$ through the magnitude of $\langle\bX, \bbeta^*\rangle$ instead of the value of $\langle\bX, \bbeta^*\rangle$. Also, it is more difficult than the original phase retrieval problem due to the additional quantization. Under our framework, the model in \eqref{eq:w02} corresponds to the link function  
\begin{equation}
\label{model:one-bit PR}
f(z) = \sign(|z| - \theta).
\end{equation}
It is worth noting that, unlike previous models, here $f$ is neither odd nor monotonic. For simplicity, in this paper we assume the thresholding $\theta$ is known. 

%% file: main_result.tex

\section{Main results}
In this section, we present the proposed procedure for estimating $\bbeta^*$ and the corresponding main results, both for the classical, or low dimensional setting where $p \ll n$, as well as the high dimensional setting where we assume $\bbeta^*$ is sparse, and accordingly have $p \gg n$. We first introduce a second moment estimator based on pairwise differences. We prove that the eigenstructure of the constructed second moment estimator encodes the information of $\bbeta^*$. We then propose algorithms to estimate $\bbeta^*$ based upon this second moment estimator. In the high dimensional setting where $\bbeta^*$ is sparse, computing the top eigenvector of our pairwise-difference matrix reduces to computing a sparse eigenvector.

For both low dimensional and high dimensional settings, we prove bounds on the sample-complexity and error-rates achieved by our algorithm. We then derive the minimax lower bound for the estimation of $\bbeta^*$. In both cases, we show that our error rate is minimax optimal, thereby establishing the optimality of the proposed procedure for a broad model class. For the high dimensional setting, however, our rate of convergence {\em is a local one}, which means that it holds only after we have a point that is close to the optimal solution. We also, therefore, give a bound on the sample complexity required to find a point close enough; based on recent results on sparse PCA \citep{berthet2013lowerSparsePCA}, it is widely believed that this is the best possible for polynomial-time methods.


\subsection{Conditions for success}
\label{ssec:conditions}

We now introduce several key quantities, which allow us to state precisely the conditions required for the success of our algorithm.   

\begin{definition}
For any (unknown) link function, $f$, define the quantity $\phi(f)$ as follows:
\begin{equation}
	\label{def:phi}
	\phi(f) := \mu_1^2 -\mu_0\mu_2 + \mu_0^2.
	\end{equation}
where $\mu_0$, $\mu_1$ and $\mu_2$ are given by
	\begin{equation}
	\label{mu}
	\mu_k := \EE\bigl[f(Z)Z^k\bigr],\quad  k = 0,1,2 \ldots,
	\end{equation}
where $Z \sim \cN(0,1)$. 	
\end{definition}

As we discuss in detail below, the key condition for success of our algorithm is $\phi(f) \not=0$. As we show below, this is a relatively mild condition, and in particular, it is satisfied by the three examples introduced in Section \ref{sec:example_models}. In fact, if $f$ is odd and monotonic (as in logistic regression and one-bit compressed sensing), by \eqref{mu} it always holds that $\mu_0 = 0$, which further implies that $\phi(f) = \mu_1^2$. According to \eqref{mu}, in this case we have $\mu_1 = 0$ if and only if $f(z) = 0$ for all $z \in \RR$. In other words, as long as $f(z)$ is not zero for all $z$, we have $\phi(f) > 0$. Of course, if $f(z)=0$ for all $z$, no procedure can recover $\bbeta^*$ as $Y$ is independent of $\bX$. For one-bit phase retrieval, Lemma \ref{lem:example models} shows that $\phi(f) > 0$ when the threshold $\theta$ in \eqref{eq:w02} satisfies $\theta > \theta_{\rm m}$, where $\theta_{\rm m}$ is the median of $|Z|$ with $Z \sim \cN(0,1)$, and $\phi(f) < 0$ when $\theta < \theta_{\rm m}$.  We note, in particular, that our condition $\phi(f) \not=0$ does not preclude $f$ from being discontinuous, non-invertible, or even or odd.

\subsection{Second moment estimator}
\label{SME}
We describe a novel moment estimator that enables our algorithm. Let $\{(y_i,\xb_i)\}_{i=1}^{n}$ be the $n$ i.i.d. observations of $(Y,\bX)$. Assuming without loss of generality that $n$ is even, we consider the following key transformation 
\#\label{eq:w91}
\Delta y_i := y_{2i} - y_{2i-1},\quad \Delta \x_i := \x_{2i} - \x_{2i-1}, 
\#
for $i = 1,2,...,n/2$. Our procedure is based on the following second moment 
\begin{equation}
\label{constructM}
\Mb := \frac{2}{n}\sum_{i = 1}^{n/2}\Delta y_i^2\Delta \x_i\Delta \x_i^{\top} \in \RR^{p \times p}.
\end{equation}
It is worth noting that constructing $\Mb$ using the differences between all pairs of $\xb_i$ and $y_i$ instead of the consecutive pairs in \eqref{eq:w91} yields similar theoretical guarantees. However, this significantly increases the computational complexity for calculating $\Mb$ when $n$ is large. 

The intuition behind this second moment is as follows. By \eqref{model}, the variation of $\bX$ along the  direction $\bbeta^*$ has the largest impact on the variation of $\la \bX, \bbeta^*\ra$. Thus, the variation of $Y$ directly depends on the variation of $\bX$ along $\bbeta^*$. Consequently, $\{(\Delta y_i,\Delta \xb_i)\}_{i=1}^{n/2}$ encodes the information of such a dependency relationship. In particular, $\Mb$ defined in \eqref{constructM} can be viewed as the covariance matrix of $\{\Delta \xb_i\}_{i=1}^{n/2}$ weighted by $\{\Delta y_i\}_{i=1}^{n/2}$. Intuitively, the leading eigenvectors of $\Mb$ correspond to the directions of maximum variations within $\{(\Delta y_i,\Delta \xb_i)\}_{i=1}^{n/2}$, which further reveals information on $\bbeta^*$. In the following, we make this intuition more rigorous by analyzing the eigenstructure of $\EE(\Mb)$ and its relationship with $\bbeta^*$. 

\begin{lemma}
	\label{M_expect}
	For $\bbeta^* \in \SSS^{p-1}$, we assume that $(Y,\bX) \in \{-1,1\}\times\RR^p$ satisfies \eqref{model}. For $\bX \sim \cN(\zero, \Ib_p)$, we have 
	\begin{equation}
	\label{eq:5}
	\EE(\Mb) = 4\phi(f) \cdot \bbeta^* \bbeta^{*\top} + 4(1 - \mu_0^2)\cdot\Ib_p, 
	\end{equation}
	where $\mu_0$ and $\phi(f)$ are defined in \eqref{mu} and \eqref{def:phi}.
\end{lemma}
\begin{proof}
	Let $\bX$ and $\bX'$ be two independent random vectors following $\cN(\zero,\Ib_p)$. Let $Y$ and $Y'$ be two binary responses that depend on $\bX,\bX'$ via \eqref{model}. Then we have
	\$
	\EE(\Mb) = \EE \bigl[ (Y - Y')^2(\bX - \bX')(\bX - \bX')^{\top} \bigr].
	\$ 
	Note that $(Y - Y')^2$ is a binary random variable taking values in $\{0,4\}$. We have 
	\#
	& \EE\bigl[ (Y-Y')^2|\bX = \xb, \bX' = \xb'\bigr] = 4 \cdot \PP\bigl[(Y - Y')^2 = 4| \bX = \xb, \bX' = \xb'\bigr] \notag\\
	& = 4 \cdot \PP( Y = 1| \bX = \xb ) \cdot \PP(Y' = -1| \bX' = \xb') + 4 \cdot \PP( Y' = 1|  \bX' = \xb' ) \cdot\PP( Y = -1| \bX = \xb ) \notag\\
	&= 2 - 2  f(\langle\x,\bbeta^*\rangle) f(\langle\x',\bbeta^*\rangle).
	\#
	There exists some rotation matrix $\Qb  \in  \RR^{p \times p}$ such that $\Qb\bbeta^*  =  \e_1  :=  [1,0, \ldots,0]^\top$. Let $\overbar{\bX}  :=  \Qb\bX$ and $\overbar{\bX}'  :=  \Qb\bX'$. Then we have
	\begin{align*}
	\EE\bigl[ (Y-Y')^2|\bX = \xb, \bX' = \xb'\bigr] = \EE\big[(Y - Y')^2| \overbar{\bX} = \Qb\xb, \overbar{\bX}' = \Qb\xb' ]= 2 - 2\cdot f(\overbar{x}_1) \cdot f(\overbar{x}'_1), 
	\end{align*}
	where $\overbar{x}_1$ and $\overbar{x}'_1$ denote the first entries of $\overbar{\xb} := \Qb\xb$ and $ \overbar{\xb}' := \Qb\xb'$ respectively. Note that $\overbar{\bX}$ and $\overbar{\bX}'$  also follow $\cN(\zero, \Ib_p)$ since symmetric Gaussian distribution is rotation invariant. Then we have
	\$
	\EE(\Mb) & = \EE \left\{ \big[2 - 2f\big(\overbar{\bX}_1\big)f\big(\overbar{\bX}_1'\big)\big](\bX - \bX')(\bX - \bX')^{\top} \right\} \\
	& =  \Qb^{\top}\EE \left\{ \big[2 - 2f(\overbar{\bX}_1)f(\overbar{\bX}_1')\big](\overbar{\bX} -\overbar{\bX}')(\overbar{\bX} -\overbar{\bX}')^{\top} \right\} \Qb \\
	& = 4\Qb^{\top} \bigl[(\mu_1^2 -\mu_0\mu_2 + \mu_0^2)\cdot \e_1 \e_1^{\top} + (1 - \mu_0^2)\cdot \Ib_p\bigr] \Qb = 4\phi(f)\cdot\bbeta^*\bbeta^{*\top} + 4(1 - \mu_0^2)\cdot \Ib_p.
	\$
	The third equality is from the definitions of $\mu_0,\mu_1,\mu_2$ in \eqref{mu} and the last equality is from \eqref{def:phi}. 
\end{proof}

Lemma \ref{M_expect} proves that $\bbeta^*$ is the leading eigenvector of $\EE(\Mb)$ as long as the eigengap $\phi(f)$ is positive. If instead we have $\phi(f) < 0$, we can use a related moment estimator which has analogous properties. To this end, define:

\begin{equation}
\label{M_alter}
\Mb' := \frac{2}{n}\sum_{i = 1}^{n/2}(y_{2i}+y_{2i-1})^2 \Delta\x_i\Delta\x_i^{\top}.
\end{equation}
In parallel to Lemma \ref{M_expect}, we have a similar result for $\Mb'$ as stated below.
\begin{corollary}
	\label{cor:M_expect}
	For $\phi(f)$ defined in \eqref{def:phi} and $\Mb'$ defined in \eqref{M_alter}, we have
	\$
	\EE(\Mb') =  -4\phi(f)\cdot\bbeta^*\bbeta^{*\top} + 4(1+\mu_0^2) \cdot \Ib_p.
	\$
\end{corollary}
\begin{proof}
	The proof is similar to that of Lemma \ref{M_expect}. Using the same notation, we have
	\$
	\EE(\Mb') = \EE\big[(Y+Y')^2(\Xb - \Xb')(\Xb - \Xb')^{\top}\big].
	\$
	Note that 
	\begin{align*}
	\EE\big[ (Y+Y')^2|\bX = \xb, \bX' = \xb'\big] &= 4 \cdot\PP\big[(Y + Y')^2 = 4| \bX = \xb, \bX' = \xb' \big] \\
	&= 2 + 2f(\langle\x,\bbeta^*\rangle)f(\langle\x',\bbeta^*\rangle). 
	\end{align*}
	Then following the same proof of Lemma \ref{M_expect}, we reach the conclusion.
\end{proof}
Corollary \ref{cor:M_expect} therefore shows that when $\phi(f) < 0$, we can construct another second moment estimator $\Mb'$ such that $\bbeta^*$ is the leading eigenvector of $\EE(\Mb')$. As discussed above, this is precisely the setting for one-bit phase retrieval when the quantization threshold in \eqref{def:phi} satisfies $\theta < \theta_{\rm m}$.  For simplicity of the discussion, hereafter we assume that $\phi(f) > 0$ and focus on the second moment estimator $\Mb$ defined in \eqref{constructM}.

A natural question to ask is whether $\phi(f) \not= 0$ holds for specific models. The following lemma demonstrates exactly this, for the example models introduced in Section \ref{sec:example_models}. 
\begin{lemma}
	\label{lem:example models}
	For any $f:\R \rightarrow [-1,1]$, recall $\phi(f)$ is defined in (\ref{def:phi}). Let $C$ be an absolute constant.  
	
	\vspace{1pt}
	\noindent(a) For flipped logistic regression, the link function $f$ is defined in \eqref{model: logitRegression}. By setting the intercept to be $\zeta = 0$, we have $\phi(f) \geq  C(1-2p_{\rm e})^2$. Therefore, we obtain $\phi(f) > 0$ for $p_{\rm e} \in [0, 1/2)$. In particular, we have $\phi(f) > 0$ for the standard logistic regression model in \eqref{noisyLogit}, since it corresponds to $p_{\rm e} = 0$. 
	\vspace{1pt}
	\noindent(b) For robust one-bit compressed sensing, $f$ is defined in \eqref{model:robust one-bit CS}. Recall that $\sigma^2$ denotes the variance of the noise term $\epsilon$ in \eqref{robust one-bit CS}. We have
	\$
	\phi(f) \geq \left\{
	\begin{aligned}
		&C\left(\frac{1-\sigma^2}{1+\sigma^2}\right)^2,   &\sigma^2 < \frac{1}{2}, \\
		& \frac{C'\sigma^4}{(1+\sigma^3)^2},   &\sigma^2 \geq \frac{1}{2}.\\
	\end{aligned} \right.
	\$
	
	\vspace{1pt}
	\noindent(c) For one-bit phase retrieval, $f$ is defined in \eqref{model:one-bit PR}. For $Z \sim \cN(0,1)$, we define $\theta_{\rm m}$ to be the median of $|Z|$, i.e., $\PP(|Z| \geq \theta_{\rm m}) = 1/2$. We have
	$|\phi(f)| \geq C\theta|\theta - \theta_{\rm m}|\exp(-\theta^2)$ and $\sign[\phi(f)] = \sign(\theta - \theta_{\rm m})$. Therefore, we obtain $\phi(f) > 0$ for $\theta > \theta_{\rm m}$. 
\end{lemma}
\begin{proof}
	See \S \ref{proof:lem:example models} for a detailed proof.
\end{proof}

\subsection{Low dimensional recovery}
We consider estimating $\bbeta^*$ in the classical (low dimensional) setting where $p \ll n$.  Based on the second moment estimator $\Mb$ defined in (\ref{constructM}), estimating $\bbeta^*$ amounts to solving a noisy eigenvalue problem.  We solve this by a simple iterative algorithm: provided an initial vector $\bbeta^0 \in \SSS^{p-1}$ (which may be chosen at random) we perform power iterations as shown in Algorithm \ref{alg:low}. The performance of Algorithm \ref{alg:low} is characterized in the following result.

\begin{algorithm}[!htb]
	\caption{Low dimensional Recovery}
	\label{alg:low}
	\begin{algorithmic}[1]
		\INPUT  $\{(y_i, \x_i)\}_{i=1}^{n}$, number of iterations $T_{\max}$
		\STATE \textbf{Second moment estimation:} $\Mb \leftarrow 2/n \cdot \sum_{i=1}^{n/2}(y_{2i} {-} y_{2i-1})^2(\x_{2i} {-} \x_{2i-1})(\x_{2i} {-} \x_{2i-1})^\top$
		\STATE \textbf{Initialization:} Choose a random vector $\bbeta^0 \in \mathbb{S}^{n-1}$
		\FOR{$t = 1, 2, \ldots, T_{\max}$}
		\STATE $\bbeta^{t} \leftarrow \Mb\cdot\bbeta^{t-1}$
		\STATE $\bbeta^{t} \leftarrow \bbeta^{t}/\|\bbeta^{t}\|_2$
		\ENDFOR
		\OUTPUT $\bbeta^{T_{\max}}$
	\end{algorithmic}
\end{algorithm}
\begin{theorem} 
	\label{thm:low}
	We assume $\bX \sim \cN(\zero, \Ib_p)$ and $(Y,\bX)$ follows \eqref{model}. Let $\{(y_i,\xb_i)\}_{i=1}^{n}$ be $n$ i.i.d. samples of response input pair $(Y,\bX)$. For any link function $f$ in \eqref{model} with $\mu_0, \phi(f)$ defined in \eqref{mu} and \eqref{def:phi}, and $\phi(f) > 0$\footnote{Recall that we have an analogous treatment and thus results for $\phi(f) < 0$.}. We let
	\#\label{eq:w821}
	\gamma := \left[\frac{1-\mu_0^2}{\phi(f)+1-\mu_0^2} + 1 \right] \bigg/2, \quad \text{and~~} \xi := \frac{\gamma\phi(f) + (\gamma-1)(1 - \mu_0^2)}{(1+\gamma)\bigl[\phi(f) + 1 - \mu_0^2\bigr]}. 
	\#
	There exist constant  $C_i$ such that when $n \geq C_1{p}/{\xi^2}$, we have that with probability at least $1 - 2\exp(-C_2p)$,
	\begin{equation}
	\label{upperbound_low}
	\big\|\bbeta^{t} - \bbeta^*\big\|_2 \leq \underbrace{C_3\cdot \frac{\phi(f)+1-\mu_0^2}{\phi(f)}\cdot \sqrt{\frac{p}{n}}}_{\dr Statistical~Error} + \underbrace{\sqrt{\frac{1 - \alpha^2}{\alpha^2}}\cdot \gamma^t}_{\dr Optimization~Error},\quad \text{for~~} t = 1,\ldots, T_{\max}. 
	\end{equation}
	Here $\alpha = \big\langle \bbeta^0, \widehat{\bbeta} \bigr\rangle$, where $\widehat{\bbeta}$ is the first leading eigenvector of $\Mb$. 
\end{theorem}
\begin{proof}
	See \S \ref{proof:thm:low} for detailed proof.
\end{proof}
Note that by \eqref{eq:w821} we have $\gamma\in(0,1)$. Thus, the optimization error term in \eqref{upperbound_low} decreases at a geometric rate to zero as $t$ increases. For $T_{\max}$ sufficiently large such that the statistical error and optimization error terms in \eqref{upperbound_low} are of the same order, we have 
\$
\bigl\|\bbeta^{T_{\max}}  - \bbeta^*\bigr\|_2  \lesssim \sqrt{p/n}.
\$
This statistical rate of convergence matches the rate of estimating a $p$-dimensional vector in linear regression without any quantization, and will later be shown to be optimal. This result shows that the lack of prior knowledge on the link function and the information loss from quantization do not keep our procedure from obtaining the optimal statistical rate. The proof of Theorem \ref{thm:low} is based on a combination of the analysis for the power method under noisy perturbation and a concentration analysis for $\Mb$. It is worth noting that the concentration analysis is close to but different from the one used in principal component analysis (PCA) since $\Mb$ defined in \eqref{constructM} is not a sample covariance matrix. 

{\bf Implications for example models.} We now apply Theorem \ref{thm:low} to specific models defined in \S \ref{sec:example_models} and quantify the corresponding statistical rates of convergence.  
\begin{corollary}
	Under the settings of Theorem \ref{thm:low}, we have the following results for a sufficiently large $T_{\max}$. 
	\begin{itemize}
		\item Flipped logistic regression: For any $p_{\rm e} \in [0,{1}/{2})$, it holds that 
		\$
		\bigl\|\bbeta^{T_{\max}}  - \bbeta^*\bigr\|_2 \lesssim \max\bigg\{1, \frac{1}{(1-2p_{\rm e})^2}\bigg\} \cdot \sqrt{\frac{p}{n}}.
		\$ 
		For $p_{\rm e} = 0$, it implies the result for the standard logistic regression model. 
		\item Robust one-bit compressed sensing:  For any $\sigma \geq 0$, it holds that 
		\$
		\bigl\|\bbeta^{T_{\max}}  - \bbeta^*\bigr\|_2 \lesssim \max\big\{1, \sigma^2\big\} \cdot \sqrt{\frac{p}{n}}.
		\$
		For $\sigma = 0$, it implies the result for standard one-bit compressed sensing. 
		\item One-bit phase retrieval: For any threshold that satisfies $\theta > \theta_{\rm m}$, where $\theta_{\rm m}$ is a constant defined in Lemma \ref{lem:example models}, it holds that 
		\$
		\bigl\|\bbeta^{T_{\max}}  - \bbeta^*\bigr\|_2 \lesssim \max\bigg\{1, \frac{1}{\theta(\theta-\theta_{\rm m})\exp(-\theta^2)}\bigg\} \cdot \sqrt{\frac{p}{n}}.
		\$
	\end{itemize}
\end{corollary}
\begin{proof}
	These results follow from combining Lemma \ref{lem:example models} and Theorem \ref{thm:low}.
\end{proof}

\subsection{High dimensional recovery}
\label{ssec:highdim}
Next we consider the high dimensional setting where $p \gg n$ and $\bbeta^*$ is sparse, i.e., $\bbeta^* \in \SSS^{p-1} \cap \mathbb{B}_0 (s,p)$ with $\mathbb{B}_0(s,p)$ defined in \eqref{eq:w9182} and $s$ being the sparsity level. Although this high dimensional estimation problem is closely related to the well-studied sparse PCA problem, the existing works \citep{Zou06, shen2008sparse, d2008optimal, witten2009penalized, journee2010generalized, yuan2011truncated, ma2011sparse, vu2013fantope, cai2013sparse} on sparse PCA do not provide a direct solution to our problem. In particular, they either lack statistical guarantees on the convergence rate of the obtained estimator \citep{shen2008sparse, d2008optimal, witten2009penalized, journee2010generalized} or rely on the properties of the sample covariance matrix of Gaussian data \citep{cai2013sparse, ma2011sparse}, which are violated by the second moment estimator defined in \eqref{constructM}. For the sample covariance matrix of sub-Gaussian data, \cite{vu2013fantope} prove that the convex relaxation proposed by \cite{d2004direct} achieves a suboptimal $s\sqrt{\log p/n}$ rate of convergence. \cite{yuan2011truncated} propose the truncated power method, and show that it attains the optimal $\sqrt{s\log p/n}$ rate {\em locally}; that is, it exhibits this rate of convergence only in a neighborhood of the true solution where $\langle \bbeta^0, \bbeta^*\rangle > C$ where $C>0$ is some constant. It is well understood that for a random initialization on $\mathbb{S}^{p-1}$, such a condition fails with probability going to one as $p \rightarrow \infty$. 

Instead, we propose a two-stage procedure for estimating $\bbeta^*$ in our setting. In the first stage, we adapt the convex relaxation proposed by \citet{vu2013fantope} and use it as an initialization step, in order to obtain a good enough initial point satisfying the condition $\langle \bbeta^0, \bbeta^*\rangle > C$. Then we adapt the truncated power method. This procedure is illustrated in Algorithm \ref{alg:high}. The initialization phase of our algorithm requires $O(s^2 \log p)$ samples (see below for more precise details) to succeed. As work in \cite{berthet2013lowerSparsePCA} suggests, it is unlikely that a polynomial time algorithm can avoid such dependence. However, once we are near the solution, as we show, this two-step procedure achieves the optimal error rate of $\sqrt{s \log p/n}$. 

\begin{algorithm}[!htb]
	\caption{Sparse Recovery}
	\label{alg:high}
	\begin{algorithmic}[1]
		\INPUT $\{(y_i, \x_i)\}_{i=1}^{n}$, number of iterations $T_{\max}$, regularization parameter $\rho$, sparsity level $\hat{s}$.
		\STATE \textbf{Second moment estimation:} $\Mb \leftarrow 2/n \cdot \sum_{i=1}^{n/2}(y_{2i} {-} y_{2i-1})^2(\x_{2i} {-} \x_{2i-1})(\x_{2i} {-} \x_{2i-1})^\top$ 
		\STATE \textbf{Initialization:} 
		\STATE \hspace{0.10in} $\bPi^0 \leftarrow \argmin_{\bPi\in\RR^{p {\times} p}}\left\{  -\la \Mb, \bPi \ra + \rho \| \bPi \|_{1,1} | \Tr(\bPi) = 1, \zero \preceq \bPi \preceq \Ib \right\}$  \label{convex-relax}
		\STATE \hspace{0.10in} $\overbar{\bbeta}^0 \leftarrow$ first leading eigenvector of $\bPi^0$\label{eigen-vec}
		\STATE \hspace{0.10in} $\cIs^0 \leftarrow \text{the set of index } j\text{'s} \text{~with the top } \hat{s} \text{~largest } |\overbar{\beta}_j^0|\text{'s}$ \label{truncate:row}
		\STATE \hspace{0.10in} \textbf{For} $j \in \{1, \ldots, p\}$
		\STATE \hspace{0.24in} $\beta^{0}_j \leftarrow \ind{\{j \in \cIs^0\}} \cdot \overbar{\beta}^{0}_j$
		\STATE \hspace{0.10in} \textbf{end For}
		\STATE \hspace{0.10in} $\bbeta^{0} \leftarrow \bbeta^{0}/\|\bbeta^{0}\|_2$
		\FOR{$t = 1, 2, \ldots, T_{\max}$}
		\STATE $\bbeta^{t} \leftarrow \Mb\cdot\bbeta^{t-1}$
		\STATE $\cIs^t \leftarrow \text{the set of index } j{'s} \text{~with the top } \hat{s} \text{~largest } |\beta_j^t|\text{'s}$ \label{truncate:row1}
		\FOR{$j \in \{1, \ldots, p\}$}
		\STATE $\beta^{t}_j \leftarrow \ind{\{j \in \cIs^t\}} \cdot \beta^{t}_j$
		\ENDFOR
		\STATE $\bbeta^{t} \leftarrow \bbeta^{t}/\|\bbeta^{t}\|_2$
		\ENDFOR
		\OUTPUT $\bbeta^{T_{\max}}$
	\end{algorithmic}
\end{algorithm}

We discuss the specific details of Algorithm \ref{alg:high}. The initialization $\bbeta^0$ is obtained by solving the convex minimization problem in line 3 of Algorithm \ref{alg:high} and then conducting an eigenvalue decomposition. The convex minimization problem is a relaxation of the original sparse PCA problem, $\max_{\bbeta \in \SSS^{p-1} \cap \mathbb{B}_0(s,p)} \bbeta^\top \Mb \bbeta$ (see \cite{d2004direct} for details). In line 3, $\rho > 0$ is the regularization parameter and $\| \cdot \|_{1,1}$ denotes the sum of the absolute values of all entries. The convex optimization problem in line 3 can be easily solved by the 
{\sl alternating direction method of multipliers} (ADMM) algorithm (see \cite{boyd2011distributed, vu2013fantope} for details). Its minimizer is denoted by $\bPi^0 \in \RR^{p \times p}$. In line 4, we set $\overbar{\bbeta}^0$ to be the first leading eigenvector of $\bPi^0$, and further perform truncation (lines 5-8) and renormalization (line 9) steps to obtain the initialization $\bbeta^0$. After this, we iteratively perform power iteration (line 11 and line 16), together with a truncation step (lines 12-15) that enforces the sparsity of the eigenvector. 

The following theorem provides simultaneous statistical and computational characterizations of Algorithm \ref{alg:high}.  

\begin{theorem}\label{thm:high} 
Let 
\begin{align}\label{def:kappa}
\kappa := \bigl[4 (1 - \mu_0^2) + \phi(f)\bigr]\big/\bigl[4(1 - \mu_0^2) + 3 \phi(f)\bigr] < 1, 
\end{align}
and the minimum sample size be 
\begin{align}\label{def:nmin}
n_{\min} := C \cdot s^2 \log p \cdot \phi(f)^2 \cdot \min \bigl\{  \kappa (1-\kappa^{1/2})/2, \kappa/8 \bigr\} \big/ \left[  (1 - \mu_0^2) + \phi(f) \right]^2. 
\end{align}
Suppose $\rho \!=\! C \bigl[\phi(f) \!+\! (1 \!-\! \mu_0^2)\bigr]\sqrt{\log p/n}$ with a sufficiently large constant $C$, where $\phi(f)$ and $\mu_0$ are specified in \eqref{mu} and \eqref{eq:5}. Meanwhile, assume the sparsity parameter $\hat{s}$ in Algorithm \ref{alg:high} is set to be $\hat{s} \!=\! C''\max\left\{\left\lceil 1/{(\kappa^{-1/2} \!-\! 1)^2} \right\rceil,\! 1\right\} \!\cdot\! s^*$. For $n \geq n_{\min}$ with $n_{\min}$ defined in \eqref{def:nmin}, we have 
	\#\label{eq:w9132}
	\|\bbeta^{t} - \bbeta^*\|_2 \leq \underbrace{C \cdot \frac{\bigl[\phi(f) + (1 - \mu_0^2)\bigr]^{\frac{5}{2}} (1 - \mu_0^2)^{\frac{1}{2}}}{\phi(f)^3} \cdot \sqrt{\frac{s \log p}{n}}}_{\dr Statistical~Error} + \underbrace{\kappa^t \cdot \sqrt{\min \bigl\{   (1-\kappa^{1/2})/2, 1/8 \bigr\}}}_{\dr Optimization~Error}
	\#
	with high probability. Here $\kappa$ is defined in \eqref{def:kappa}.  
\end{theorem}
The first term on the right-hand side of \eqref{eq:w9132} is the statistical error while the second term gives the optimization error. Note that the optimization error decays at a geometric rate since $\kappa < 1$. For $ T_{\max}$ sufficiently large, we have \$\bigl\|\bbeta^{T_{\max}}  - \bbeta^*\bigr\|_2  \lesssim \sqrt{s \log p/n}.\$
In the sequel, we show that the right-hand side gives the optimal statistical rate of convergence for a broad model class under the high dimensional setting with $p \gg n$. 

\subsection{Minimax lower bound}
We establish the minimax lower bound for estimating $\bbeta^*$ in the model defined in \eqref{model}. In the sequel we define the family of link functions that are Lipschitz continuous and are bounded away from $\pm 1$. Formally, for any $m \in (0,1)$ and $L > 0$, we define
\#\label{eq:w0109}
\cF(m,L) := \bigl\{f:   |f(z)| \leq 1 - m,  \quad |f(z) - f(z')| \leq L|z - z'|, \quad \text{for~all~}  z,z' \in \RR \bigr\}.
\#
Let $\mathcal{X}_{f}^n \!:=\! \{(y_i,\x_i)\}_{i=1}^{n}$ be the $n$ i.i.d. realizations of $(Y,\bX)$, where $\bX$ follows $\cN(\zero, \Ib_p)$ and $Y$ satisfies \eqref{model} with link function $f$. Correspondingly, we denote the estimator of $\bbeta^* \in \cB$ to be $\widehat{\bbeta}(\mathcal{X}_{f}^n)$, where $\cB$ is the domain of $\bbeta^*$. We define the minimax risk for estimating $\bbeta^*$ as 
\#
\label{minimax}
\mathcal{R}(n,m,L,\cB) := \inf_{f \in \cF(m,L)} \inf_{\widehat{\bbeta}(\mathcal{X}_{f}^n)} \sup_{\bbeta^* \in \cB} \EE \bigl\|\widehat{\bbeta}(\mathcal{X}_{f}^n) - \bbeta^* \bigr\|_2. 
\#
In the above definition, we not only take the infimum over all possible estimators $\hat{\bbeta}$, but also all possible link functions in $\cF(m,L)$. For a fixed $f$, our formulation recovers the standard definition of minimax risk \citep{yu1997assouad}. By taking the infimum over all link functions, our formulation characterizes the minimax lower bound under the least challenging $f$ in $\cF(m,L)$. In the sequel we prove that our procedure attains such a minimax lower bound for the least challenging $f$ given any unknown link function in $\cF(m, L)$. That is to say, even when $f$ is unknown, our estimation procedure is as accurate as in the setting where we are provided the least challenging $f$, and the achieved accuracy is not improvable due to the information-theoretic limit. The following theorem establishes the minimax lower bound in the high dimensional setting. Recall that $\B_0(s,p)$ defined in \eqref{eq:w9182} is the set of $s$-sparse vectors in $\RR^{p}$.

\begin{theorem} 
	\label{thm:minimax_high} 
	Let $\cB \!=\! \SSS^{p-1}\cap\B_0(s,p)$. We assume that $n \!>\! {m(1\!-\!m)}/{(2L^2)^2}\!\cdot\!\bigl[Cs\log(p/s)/2 \!-\! \log 2\bigr]$. For any $s \in (0, p/4]$, the minimax risk defined in \eqref{minimax} satisfies 
	\$
	\mathcal{R}(n,m,L,\cB) \geq C'\cdot \frac{\sqrt{m(1-m)}}{L}\cdot \sqrt{\frac{s\log(p/s)}{n}}.
	\$ 
	Here $C$ and $C'$ are absolute constants, while $m$ and $L$ are defined in \eqref{eq:w0109}.
\end{theorem}
\begin{proof}
	See \S \ref{proof:thm:minimax_high} for a detailed proof.	
\end{proof}

Theorem \ref{thm:minimax_high} establishes the minimax optimality of the statistical rate attained by our procedure for $p \!\gg\! n$ and $s$-sparse $\bbeta^*$. In particular, for arbitrary $f \in \cF(m,L) \cap \{f: \phi(f) > 0\}$, the estimator $\hat{\bbeta}$ attained by Algorithm \ref{alg:high} is minimax-optimal in the sense that its $\sqrt{s\log p/n}$ rate of convergence is not improvable, even when the information on the link function $f$ is available. The next corollary of Theorem \ref{thm:minimax_high} establishes the minimax lower bound for $p \ll n$. 
\begin{corollary}
	\label{thm:minimax_low} 
	Let $\cB = \SSS^{p-1}$. We suppose that $n > {m(1-m)}/(2L^2)\cdot (Cp - \log 2)$. The minimax risk defined in \eqref{minimax} satisfies
	\$
	\mathcal{R}(n,m,L,\cB) \geq C'\cdot \frac{\sqrt{m(1-m)}}{L}\cdot \sqrt{\frac{p}{n}},
	\$ 
	where $C$ and $C'$ are some absolute constants, while $m$ and $L$ are defined in \eqref{eq:w0109}.
\end{corollary}
\begin{proof}
	The result follows from Theorem \ref{thm:minimax_high} by setting $s = p/4$.
\end{proof}

It is worth to note that our lower bound becomes trivial for $m = 0$, i.e., there exists some $z$ such that $|f(z)| = 1$. One example is the noiseless one-bit compressed sensing model defined in \eqref{one-bit CS}, for which we have $f(z) = \sign(z)$. In fact, for noiseless one-bit compressed sensing, the $\sqrt{s \log p/n}$ rate is not optimal. For example, \citet{jacques2011robust} (Theorem 2) provide a computationally inefficient algorithm that achieves rate ${s\log p/n}$. Understanding such a rate transition phenomenon for link functions with zero margin, i.e., $m=0$ in \eqref{eq:w0109}, is an interesting future direction. 

%% file: experiment.tex

\section{Numerical results}

In this section, we provide the numerical results to support our theory. We conduct two sets of experiments. First, we examine the eigenstructures of the second moment estimators defined in \eqref{constructM} and \eqref{M_alter} for the following three models: flipped logistic regression (FLR) in \eqref{model: logitRegression}, one-bit compressed sensing with Gaussian noise $\cN(0,\sigma^2)$ (one-bit CS) in \eqref{robust one-bit CS} and one-bit phase retrieval (one-bit PR) in \eqref{eq:w02}. Second, for the same three models, we apply Algorithm \ref{alg:low} and Algorithm \ref{alg:high} to parameter estimation in the low dimensional and high dimensional regimes, respectively. Our simulations are based on synthetic data. Given $n, p$, model parameter $\bbeta^*$, and some specific model, we construct our data as follows. We first generate $n$ i.i.d. samples $\xb_1,\ldots,\x_n$ from $\cN(0,\Ib_p)$. Then for each sample $\x_i$, we generate the corresponding label $y_i$ by plugging $\big\langle \x_i, \bbeta^*\big\rangle$ into the specified model. 
\begin{figure}[ht]
	\centering
	\subfigure[Flipped Logistic Regression]{
		\centering
		\label{fig:spec_LR}
		\includegraphics[width=0.31\columnwidth]{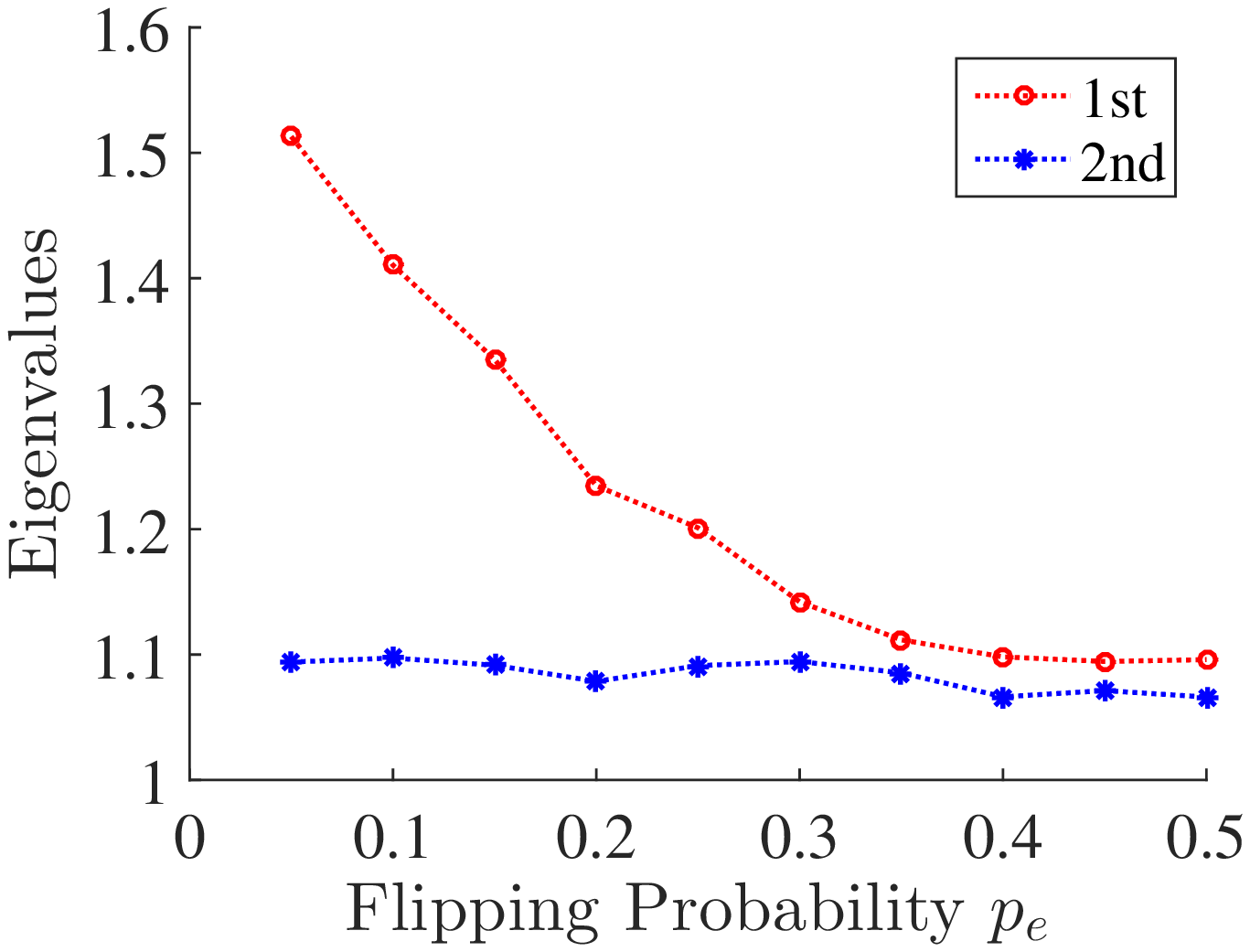}
	}
	\hfill
	\subfigure[One-bit Compressed Sensing]{
		\centering
		\label{fig:spec_CS}
		\includegraphics[width=0.31\columnwidth]{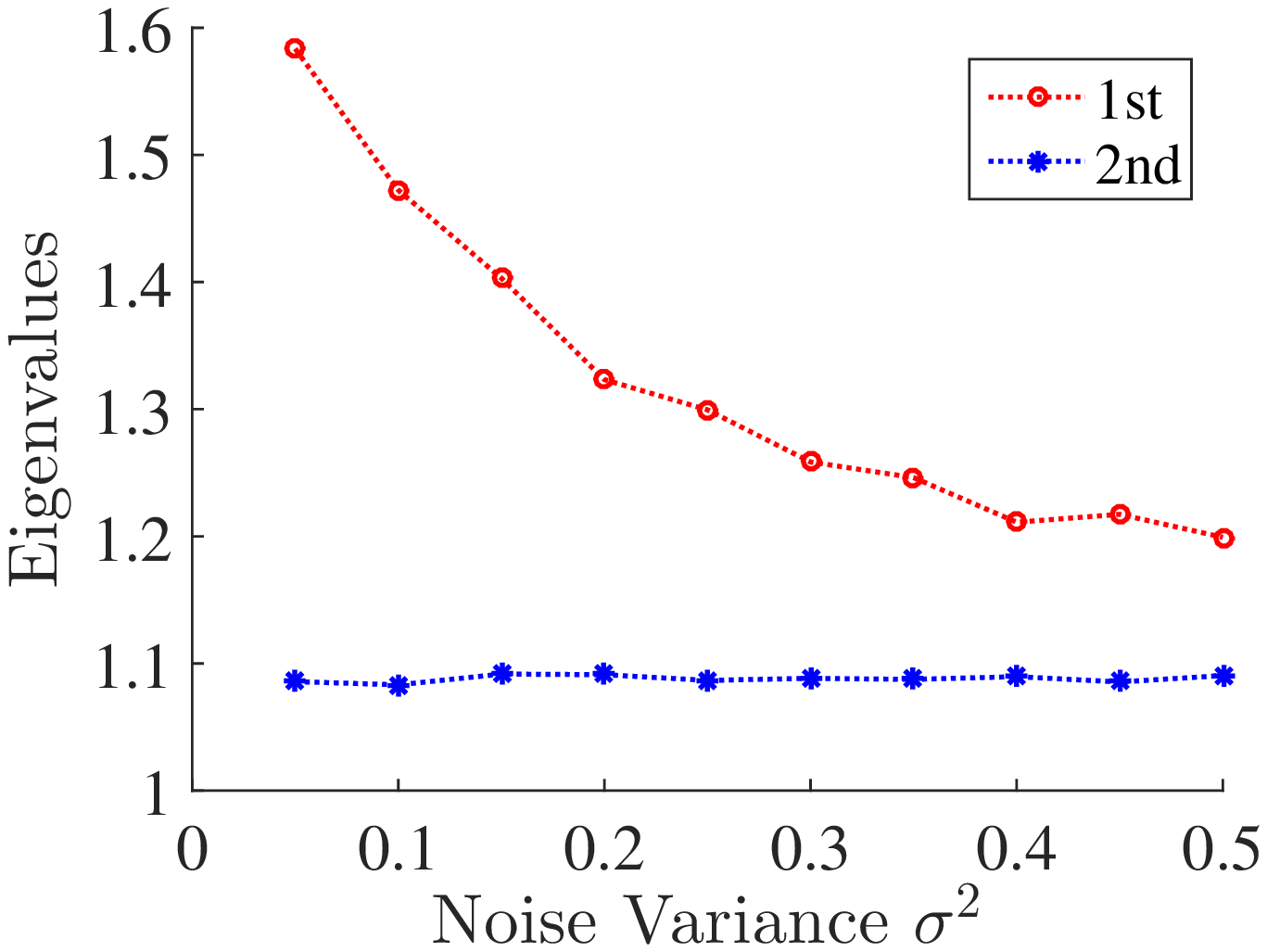}
	}
	\hfill
	\subfigure[One-bit Phase Retrieval]{
		\centering
		\label{fig:spec_PR}
		\includegraphics[width=0.31\columnwidth]{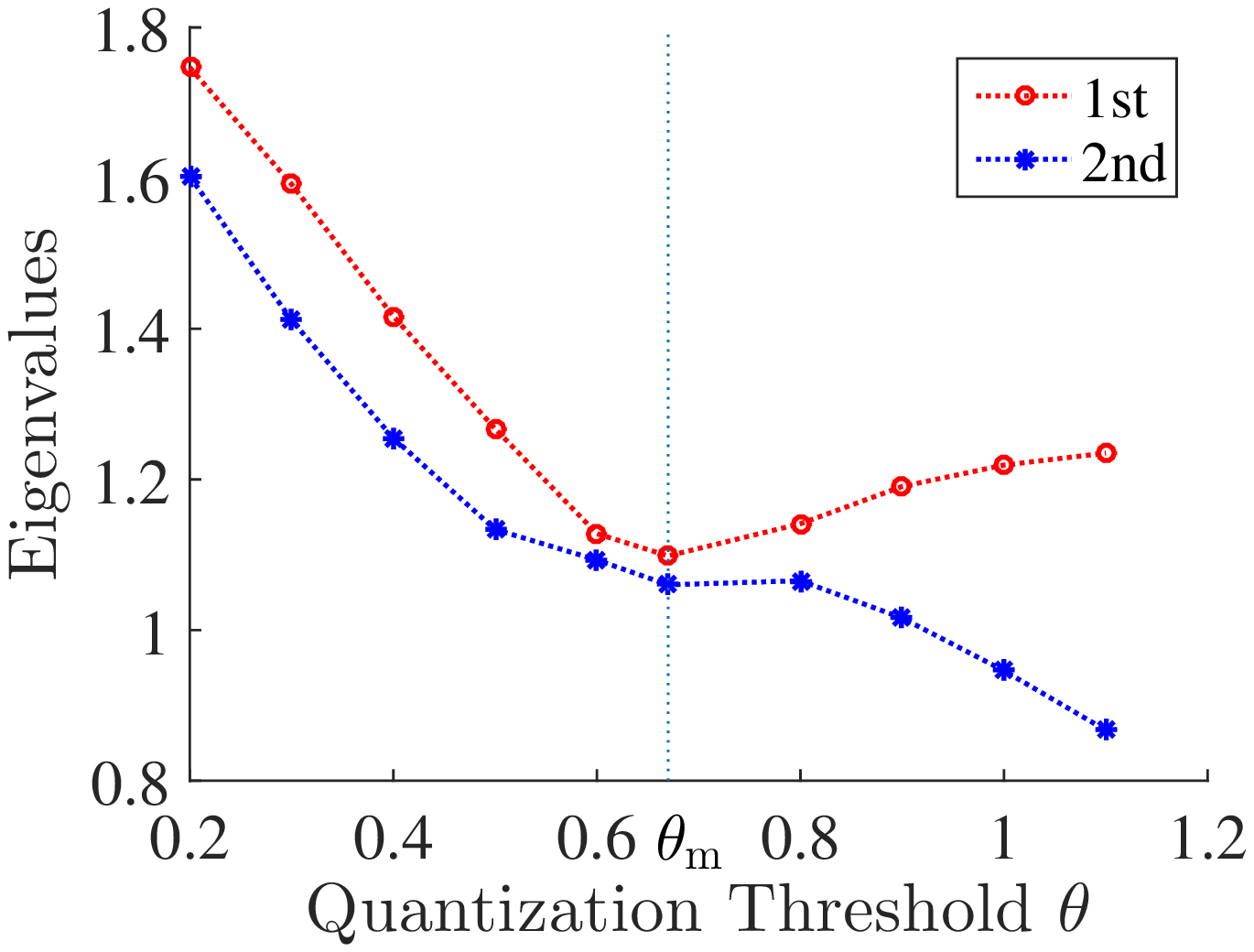}
	}
	\caption{Eigenstructure of the second moment estimators defined in \eqref{constructM} and \eqref{M_alter}. Panel (a) and (b) show the top two eigenvalues of $\Mb/4$ for FLR and 1-bit CS respectively. Panel (c) shows the top two eigenvalues of $\Mb/4$ when $\theta > \theta_{\rm m}$ and the first two eigenvalues of $\Mb'/4$ when $\theta < \theta_{\rm m}$ for 1-bit PR. }
	\label{fig:eigenstructure}
\end{figure}
For the first set of experiments, we set $n = 3000, p = 20$. We randomly select $\bbeta^*$ from $\SSS^{p-1}$. Figure \ref{fig:eigenstructure} shows the top two eigenvalues of the second moment estimator constructed from $n$ samples. Each curve is an average of $10$ independent trials.
In the first two models (FLR and 1-bit CS), as predicted by Lemmas \ref{M_expect} and \ref{lem:example models}, we observe that the gap between first two eigenvalues, corresponding to $\phi(f)$, decays with noise parameter $p_e$ and $\sigma^2$. Note that the two models have symmetric link functions, thereby we obtain $\mu_0 = 0$ and further $\mathbb{E}(\mathbb{\Mb}/4) = \phi(f)\cdot\bbeta^*\bbeta^{*\top} + \Ib_p$. This theoretical conclusion leads to the practical phenomenon that the second eigenvalue in Panel \ref{fig:spec_LR} and \ref{fig:spec_CS} stays close to $1$ and does not change with noise level. For 1-bit PR, when quantization threshold $\theta < \theta_{\rm m}$, particularly we have $\phi(f) < 0$. In this case, as claimed in Corollary \ref{cor:M_expect}, we can construct second moment estimator $\Mb'$ whose expectation has top eigenvector $\bbeta^*$ and positive eigen gap $-\phi(f)$. Panel \ref{fig:spec_PR} shows the existence of nontrivial eigen gap of $\Mb'$ in the region $\theta < \theta_{\rm m}$ thus validates our theory.
\begin{figure}[ht]
	\centering
	\subfigure[Flipped Logistic Regression]{
		\centering
		\label{fig:low_LR}
		\includegraphics[width=0.31\columnwidth]{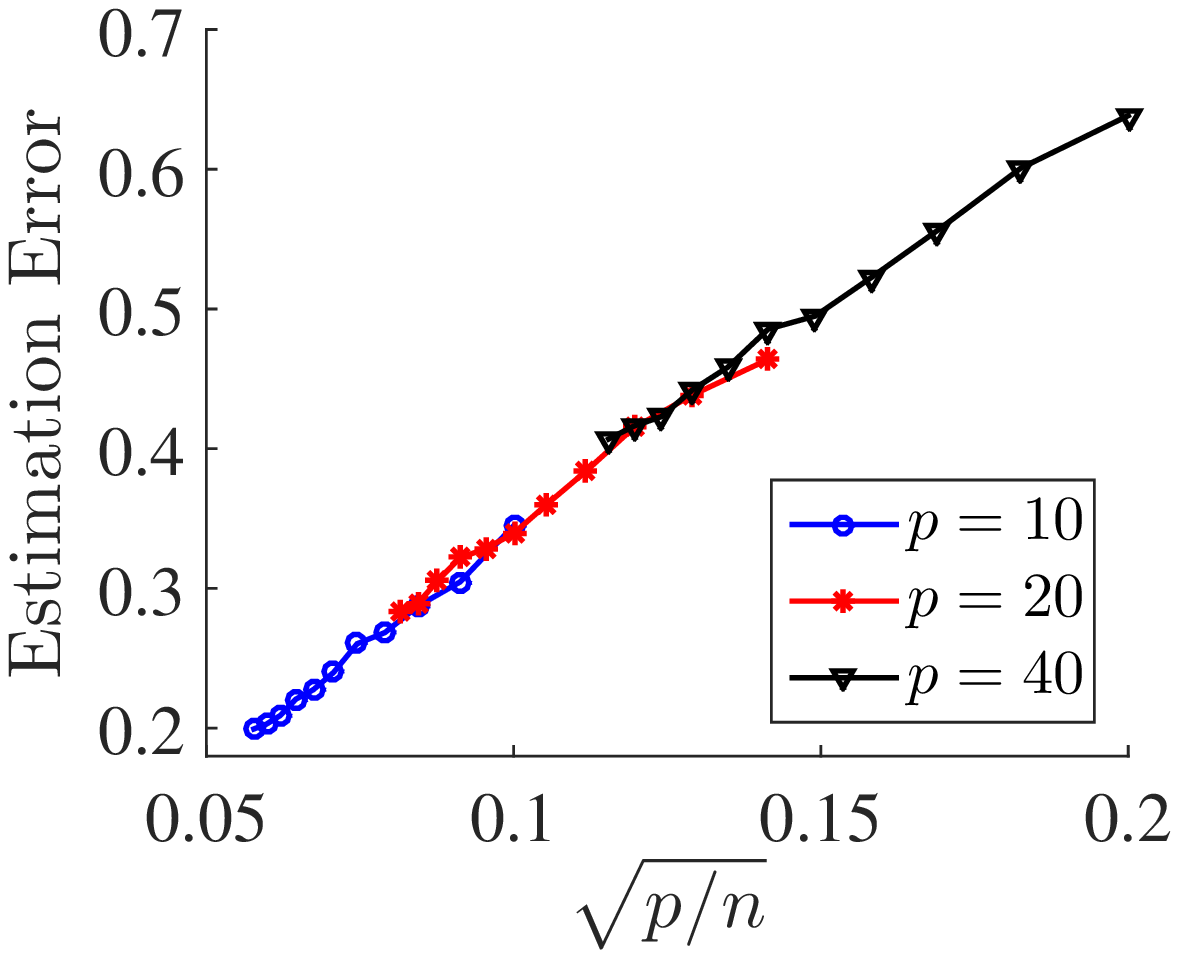}
	}
	\hfill
	\subfigure[One-bit Compressed Sensing]{
		\centering
		\label{fig:low_CS}
		\includegraphics[width=0.31\columnwidth]{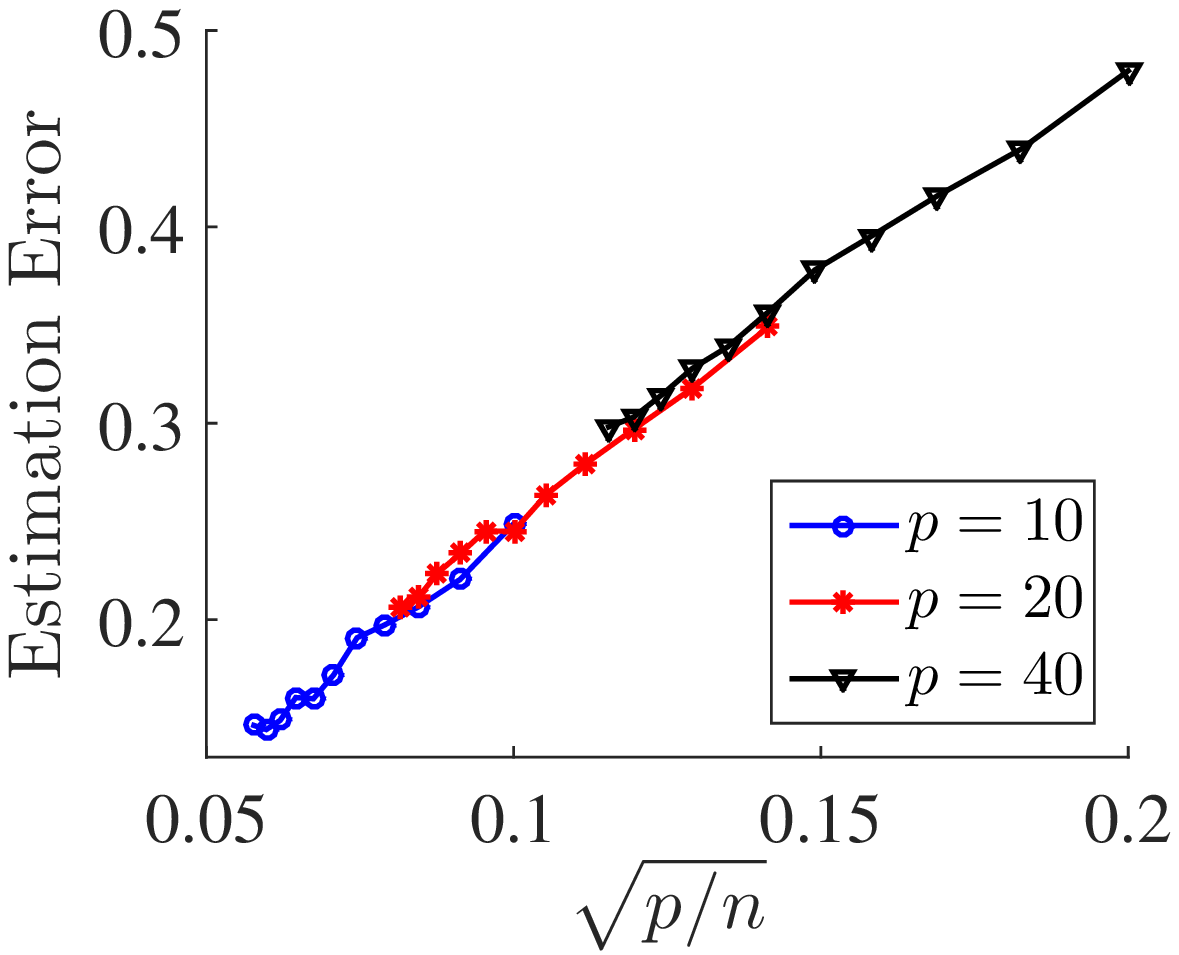}
	}
	\hfill
	\subfigure[One-bit Phase Retrieval]{
		\centering
		\label{fig:low_PR}
		\includegraphics[width=0.31\columnwidth]{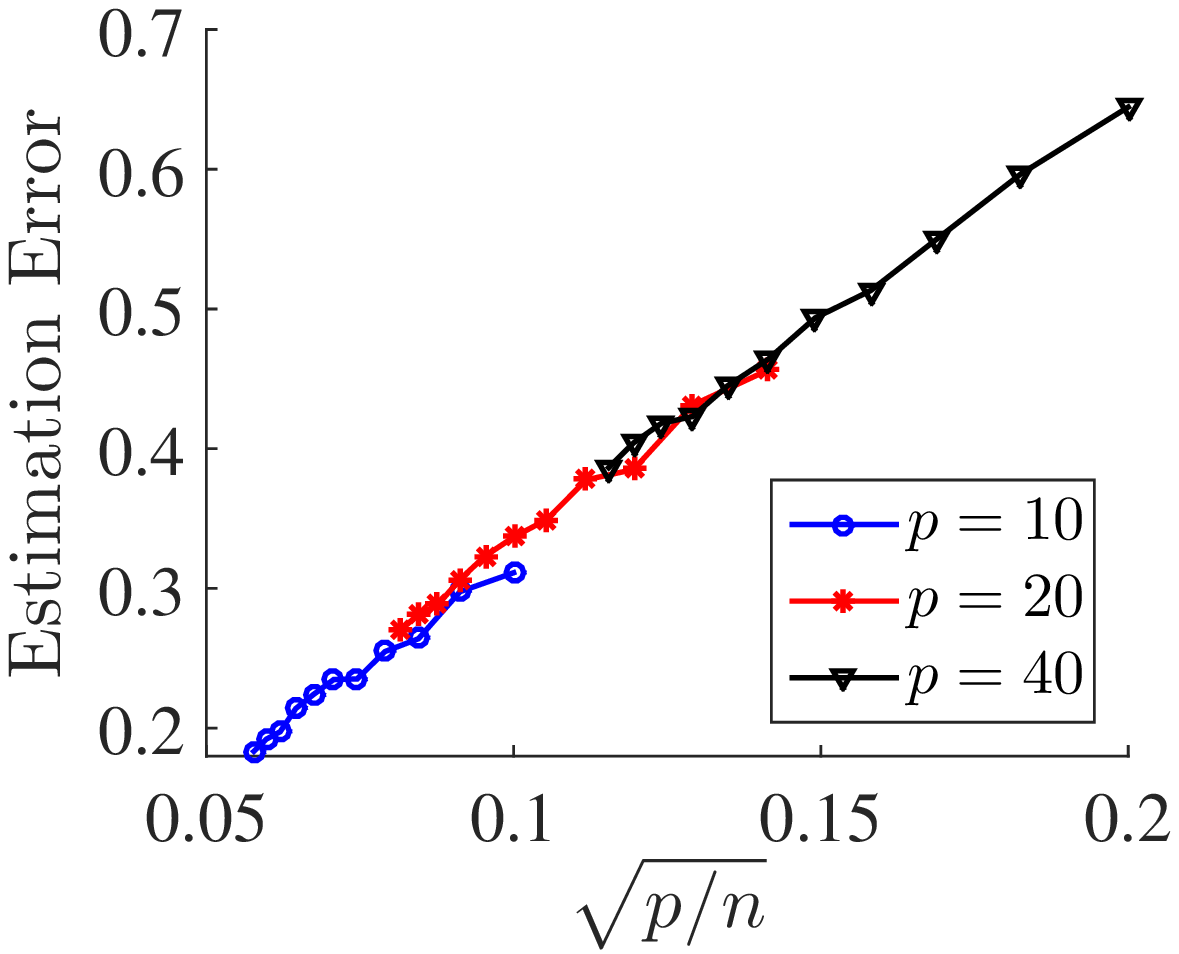}
	}
	\caption{Estimation error of low dimensional recovery in three models.  (a) For FLR, we set flipping probability $p_e = 0.1$. (b) For 1-bit CS, we set variance of Gaussian noise $\delta^2 = 0.1$. (c) For 1-bit PR, we set quantization threshold $\theta = 1$.}
	\label{fig:low}
\end{figure}

\begin{figure}[ht]
	\centering
	\subfigure[Flipped Logistic Regression]{
		\centering
		\label{fig:high_LR}
		\includegraphics[width=0.31\columnwidth]{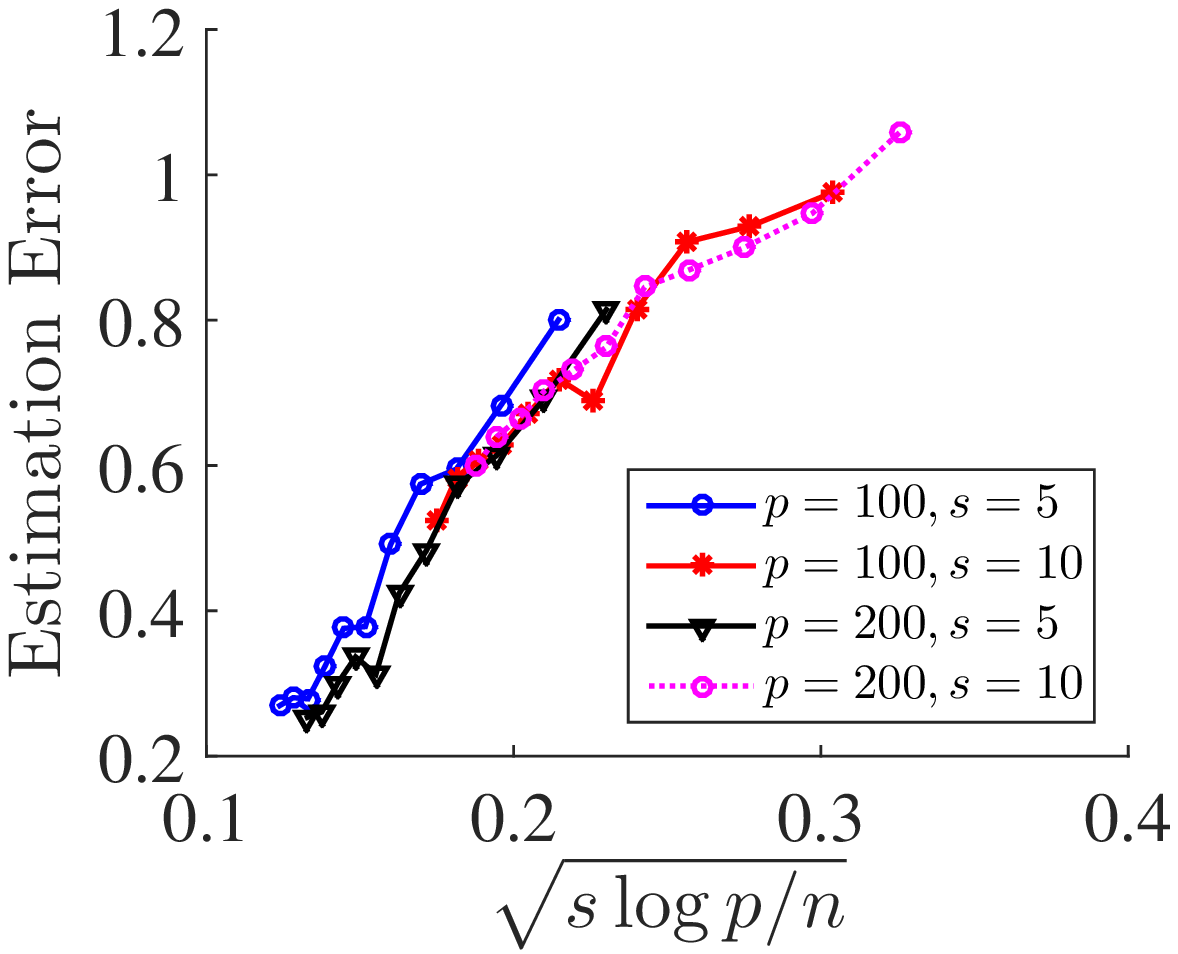}
	}
	\hfill
	\subfigure[One-bit Compressed Sensing]{
		\centering
		\label{fig:high_CS}
		\includegraphics[width=0.31\columnwidth]{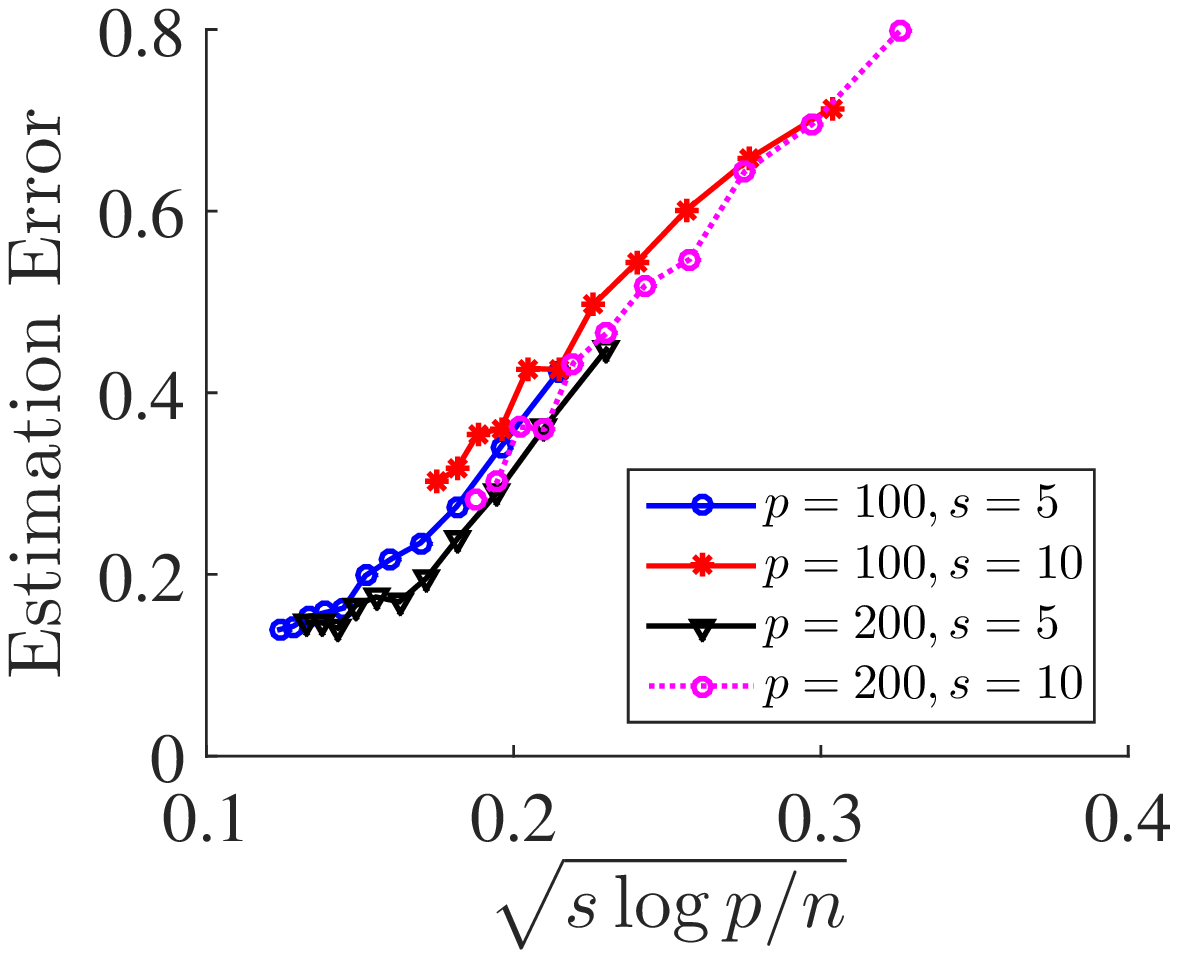}
	}
	\hfill
	\subfigure[One-bit Phase Retrieval]{
		\centering
		\label{fig:high_PR}
		\includegraphics[width=0.31\columnwidth]{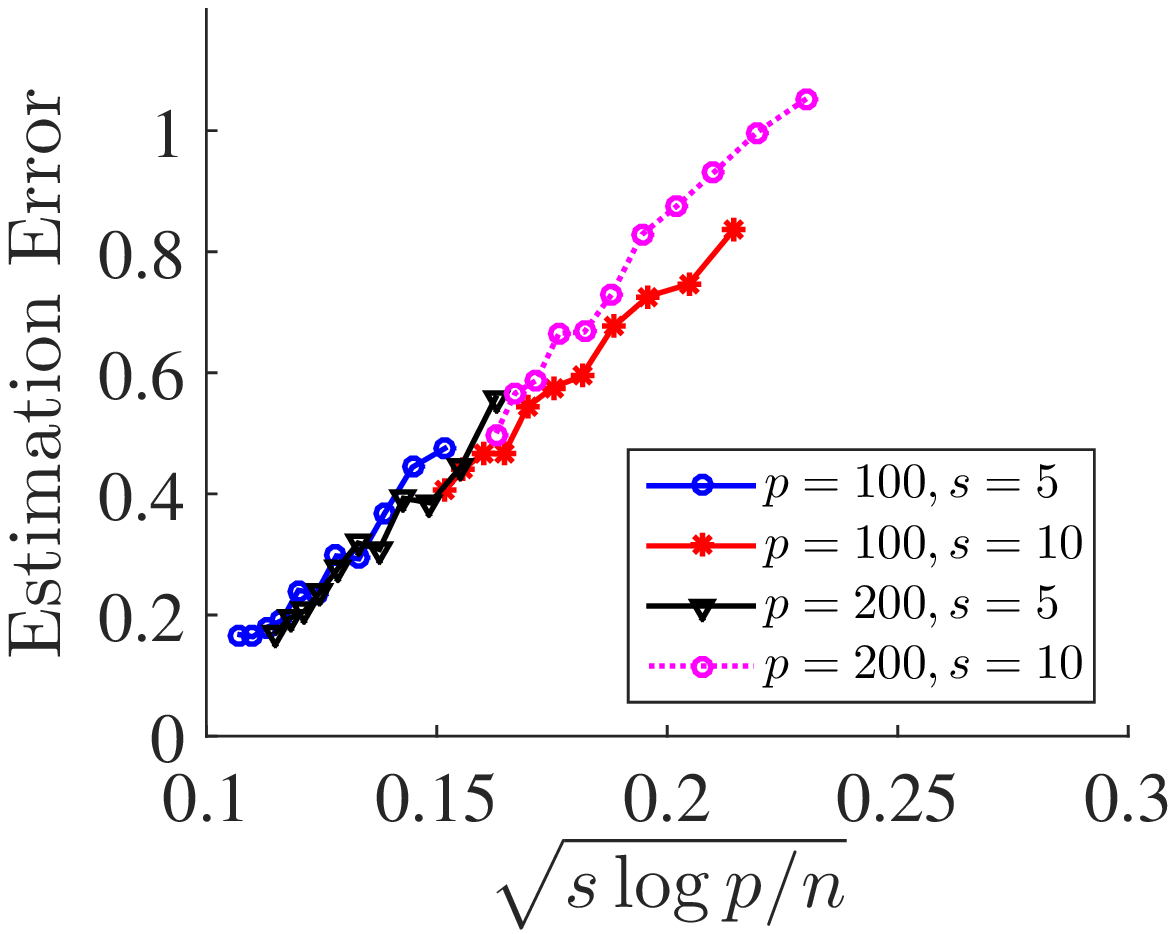}
	}
	\caption{Estimation error of sparse recovery in three models. (a) For FLR, we set flipping probability $p_e = 0.1$. (b) For 1-bit CS, we set variance of Gaussian noise $\delta^2 = 0.1$. (c) For 1-bit PR, we set quantization threshold $\theta = 1$.}
	\label{fig:high}
\end{figure}
In the second set of experiments, we fix $p_e = 0.1, \sigma^2 = 0.1, \theta = 1$ for the three models. For low dimensional recovery, we randomly select $\bbeta^*$ from $\SSS^{p-1}$. For high dimensional recovery, we generate $\bbeta^*$ as follows. Given sparsity $s$,  we first randomly select a subset $\mathcal{S}$ of $\{1,\ldots,p\}$ with size $s$ as support of $\bbeta^*$. We then set $\bbeta^*_{\mathcal{S}}$ to be a vector that is randomly generated from $\SSS^{s-1}$. We characterize the estimation error by $\ell_2$ norm. Figure \ref{fig:low} plots the estimation error against the quantity $\sqrt{p/n}$ in the low dimensional regime. Each curve is an average of $100$ independent trials. We note that for the same value of  $\sqrt{p/n}$, we obtain almost the same estimation error in practice. Moreover, we observe that the estimation error has a linear dependence on $\sqrt{p/n}$. These two empirical results correspond to our theoretical conclusion in Theorem \ref{thm:low}. Figure \ref{fig:high} plots estimation error against $\sqrt{s\log p/n}$ for recovering $s$-sparse $\bbeta^*$ with different values of $s$ and $p$. Each curve is an average of $100$ independent trials. Similar to low dimensional recovery, we observe that the estimation error is nearly proportional to $\sqrt{s\log p/n}$ and the same $\sqrt{s\log p/n}$ leads to approximately identical estimation error. This phenomenon validates Theorem \ref{thm:high}.

%% file: proof.tex
\section{Proofs}
\label{proof}
In this section, we provide the proofs for our main results. First we characterize the implications of our general framework for the models in Section \ref{sec:example_models}. We then establish the statistical convergence rates of the proposed procedure and the corresponding minimax lower bounds. 
\subsection{Proof of Lemma \ref{lem:example models}}
\label{proof:lem:example models}
\vspace{4pt}
\noindent{\bf Flipped logistic regression.} For flipped logistic regression, the link function $f$ is defined in \eqref{model: logitRegression}, where $\zeta$ is the intercept. For $\zeta = 0$, we have 
\$
f(z) = \frac{e^z-1}{e^z+1} + 2p_{\rm e} \cdot \frac{1 - e^z}{1+e^z}.
\$
Note that $f$ is odd. Hence, by \eqref{mu} we have $\mu_0 = \mu_2 = 0$. Meanwhile, from Stein's lemma, we have 
\$
\mu_1 = \EE[f'(z)] = \EE\left[ (1 - 2p_{\rm e})\cdot \frac{2e^z}{(1+e^z)^2}  \right] = (1-2p_{\rm e}) \cdot \EE\frac{2e^z}{(1+e^z)^2}.
\$
We thus have $\phi(f) = \mu_1^2 \geq C(1 - 2p_{\rm e})^2$ for some constant $C$.

\vspace{4pt}
\noindent{\bf Robust one-bit compressed sensing.} Recall in robust one-bit compressed sensing, we have 
\$
f(z) = 2\cdot \PP(z+\epsilon > 0) - 1, 
\$
where $\epsilon \sim \cN(0,\sigma^2)$ is the noise term in \eqref{robust one-bit CS}. In particular, note that  
\$
f(z) + f(-z) = 2\cdot \big[\PP(\epsilon > z)+\PP(\epsilon > -z)\big] - 2 = 0. 
\$
Hence, $f(z)$ is an odd function, which implies $\mu_0 = \mu_2 = 0$ by \eqref{mu}.  For $\mu_1$ defined in  \eqref{mu}, we have
\#\label{eq:w9133}
\mu_1 &= \EE[f(z)z]  = \EE\bigl\{ \big[2 \cdot \PP(\epsilon > -z) - 1\big] z \bigr\} = \EE\bigl[ \PP( |\epsilon| < |z|) |z| \bigr]  \geq \EE \left\{\left[1 - 2e^{-z^2/(2\sigma^2)}\right]|z|\right\} \\
& = \EE(|z|) -  \int_{-\infty}^{\infty} \frac{2}{\sqrt{2\pi}}e^{-\frac{u^2}{2\sigma^2}}e^{-\frac{u^2}{2}}|u| \ud u = \EE(|z|) \left(1 - 2\frac{\sigma^2}{1+\sigma^2} \right)  = \EE(|z|) \frac{1-\sigma^2}{1+\sigma^2}.\notag
\#
Here the inequality is from the fact that $\PP( |\epsilon| < |z|) \geq 1 - 2e^{-\frac{z^2}{2\sigma^2}}$ since $\epsilon \sim \cN(0,\sigma^2)$. For $\sigma^2 < 1/2$, we have
\$
\phi(f) = \mu_1^2 \geq C\left(\frac{1-\sigma^2}{1+\sigma^2}\right)^2, 
\$
where $C = \EE(|z|)$ with $z \sim \cN(0,1)$. For $\sigma^2 \geq 1/2$, rather than applying $\PP( |\epsilon| < |z|) \geq 1 - 2e^{-\frac{z^2}{2\sigma^2}}$ in the inequality of \eqref{eq:w9133}, we apply $\PP( |\epsilon| < |z|) \geq \frac{2}{\sqrt{2\pi}\sigma}e^{-\frac{z^2}{2\sigma^2}}|z|$ since $\epsilon \sim \cN(0,\sigma^2)$. We then obtain
\begin{align*}
\mu_1 \geq \EE\left[\frac{2}{\sqrt{2\pi}\sigma}e^{-\frac{z^2}{2\sigma^2}}z^2\right] = \frac{2}{\sqrt{2\pi}\sigma}\int_{-\infty}^{\infty} \frac{1}{\sqrt{2\pi}}e^{-\frac{u^2}{2\sigma^2}}e^{-\frac{u^2}{2}}u^2 \ud u \geq \frac{C'}{\sigma} \left({\frac{\sigma^2}{1+\sigma^2}}\right)^{\frac{3}{2}}.
\end{align*}
Finally, for $\sigma^2 \geq 1/2$ we have
\$
\phi(f) \geq \frac{C'\sigma^4}{(1+\sigma^2)^3}.
\$

\vspace{4pt}
\noindent{\bf One-bit phase retrieval.} For the one-bit phase retrieval model, the major difference from the previous two models is that $f(z)$ is even, which results in $\mu_1 = 0$. By the definition in \eqref{mu}, we have
\begin{align*} 
\mu_0  = \EE[f(z)]  = \PP(|z| \geq \theta) - \PP(|z| < \theta),
\end{align*}
and
\begin{align*}
\mu_2 = \EE\bigl[f(z)z^2\bigr] = \PP(|z| \geq \theta)\EE\bigl(z^2 \;\big|\; \abs{z} \geq \theta\bigr) - \PP(|z| < \theta)\EE\bigl(z^2~\big|~ \abs{z} < \theta\bigr).
\end{align*}
For notational simplicity, we define $p_1 = \PP(|z| \geq \theta)$. We have
\begin{align}
\label{boundofPhi}
\phi(f)  = \mu_0(\mu_0 - \mu_2) = 2p_1(2p_1 - 1)\bigl[1 - \EE\bigl(z^2 \;\big|\; \abs{z} > \theta\bigr) \bigr],
\end{align}
where the second equality follows from the fact that 
\#\label{eq:w4123}
\PP(|z| \geq \theta) + \PP(|z| < \theta) = 1,
\#
and
\#\label{eq:w4124}
\PP(|z| \geq \theta)\EE\bigl(z^2 \;\big|\; \abs{z} \geq \theta\bigr)+ \PP(|z| < \theta)\EE\bigl(z^2 \;\big|\; \abs{z} < \theta\bigr) = \EE(z^2) = 1.
\#
By \eqref{eq:w4123} and \eqref{eq:w4124} we have $p_1 > 0$ and $\EE\bigl(z^2 \;\big|\; \abs{z} \geq \theta\bigr) > 1$ for $\theta > 0$. Hence, for $\theta < \theta_{\rm m}$ with $\theta_{\rm m}$ being the median of $|z|$ with $z \sim \cN(0,1)$, we have $p_1 \geq 1/2$, which further implies $\phi(f) < 0$ by \eqref{boundofPhi}. Otherwise we have $\phi(f) > 0$. Thus, we have $\sign[\phi(f)] = \sign(\theta - \theta_{\rm m})$. 

In the following we establish a lower bound for $|\phi(f)|$. Note that
\begin{align}
\label{ep}
\EE\bigl(z^2 \;\big|\; \abs{z} \geq \theta\bigr) = \frac{2}{p_1}\int_{\theta}^{+\infty} \frac{1}{\sqrt{2\pi}}e^{-\frac{z^2}{2}}z^2dz  = \frac{2\theta}{p_1\sqrt{2\pi}}e^{-\frac{\theta^2}{2}} + 1. 
\end{align}
Plugging \eqref{ep} into (\ref{boundofPhi}) yields
\begin{equation}
\label{tmp1}
\phi(f) = -2(2p_1 - 1)\frac{2\theta}{\sqrt{2\pi}}e^{-\frac{\theta^2}{2}}.
\end{equation}
For $0 < \theta < \theta_{\rm m}$, which implies $p_1 \geq 1/2$, we have
\begin{align}
\label{ee}
p_1 - \frac{1}{2} = 2 \int_{\theta}^{\theta_{\rm m}} \frac{1}{\sqrt{2\pi}}e^{-\frac{z^2}{2}}dz \geq \frac{2}{\sqrt{2\pi}}e^{-\frac{\theta_{\rm m}^2}{2}}(\theta_{\rm m} - \theta).
\end{align}
By plugging (\ref{ep}) into (\ref{tmp1}), we have
\begin{align}
|\phi(f)| \geq \frac{8}{\sqrt{2\pi}}e^{-\frac{\theta_{\rm m}^2}{2}}(\theta_{\rm m} - \theta) \frac{2\theta}{\sqrt{2\pi}}e^{-\frac{\theta^2}{2}} \geq  C\theta(\theta_{\rm m} - \theta)e^{-\frac{\theta^2}{2}} .
\end{align}
For $\theta > \theta_{\rm m}$, which implies $p_1 < 1/2$, similarly to (\ref{ee}), we have
\begin{align}
\label{ee1}
\frac{1}{2} - p_1 = 2 \int_{\theta_{\rm m}}^{\theta} \frac{1}{\sqrt{2\pi}}e^{-\frac{z^2}{2}}dx \geq \frac{2}{\sqrt{2\pi}}e^{-\frac{\theta^2}{2}}(\theta - \theta_{\rm m}).
\end{align}
Thus, we conclude that  
\$
|\phi(f)| \geq C'\theta(\theta - \theta_{\rm m})e^{-\theta^2}. 
\$

\subsection{Proof of Theorem \ref{thm:low}}
\label{proof:thm:low}
Let $\hat{\bbeta}$ be the top eigenvector of $\Mb$ and $\hat{\lambda}_1, \hat{\lambda}_2$ be the first and second largest eigenvalues of $\Mb$. We use $\lambda_1, \lambda_2$ to denote the first and second largest eigenvalues of $\EE(\Mb)$. From Lemma \ref{M_expect}, we already know that 
\$
\lambda_1 = 4\phi(f) + 4(1 - \mu_0^2), \quad\text{and~~} \lambda_2 = 4(1 - \mu_0^2).
\$
By the triangle inequality, we have 
\$
\bigl\|\bbeta^{t} - \bbeta^*\bigr\|_2 \leq  \bigl\|\bbeta^* - \hat{\bbeta} \bigr\|_2 + \bigl\|\bbeta^{t} - \hat{\bbeta} \bigr\|_2.
\$
The first term on the right hand side is the statistical error and the second term is the optimization error. From standard analysis of the power method, we have
\$
\bigl\|\bbeta^{t} - \hat{\bbeta} \bigr\|_2 \leq \sqrt{\frac{1 - \alpha^2}{\alpha^2}} \cdot \bigl({\hat{\lambda}_2}/{\hat{\lambda}_1} \bigr)^t,
\$
where $\alpha = \bigl\langle \bbeta^0, \hat{\bbeta}\bigr\rangle$. By the definition in \eqref{constructM}, $\Mb$ is the sample covariance matrix of $n/2$ independent realizations of the random vector $(Y-Y')(\bX - \bX')\in \RR^{p}$. Since $\bX$ is Gaussian and $Y$ is bounded, $(Y-Y')(\bX - \bX')$ is sub-Gaussian. By standard concentration results (see e.g. Theorem 5.39 in \cite{vershynin2010introduction}), there some constants $C,C_1$ such that for any $t \geq 0$, with probability at least $1 - 2e^{-Ct^2}$,
\$
\|\Mb - \EE(\Mb)\|_2 \leq \max(\delta,\delta^2) \|\EE(\Mb)\|_2,
\$
where $\delta = C_1\sqrt{\frac{p}{n}} + \frac{t}{\sqrt{n}}$. We let $t = \sqrt{p}$, then for any $\xi \in (0,1)$, we have that $\|\Mb - \mathbb{E}(\Mb)\|_2 \leq \xi \|\mathbb{E}(\Mb)\|_2$ when $n \geq C_2p/\xi^2$ for sufficiently large constant $C_2$. Conditioning on $\|\Mb - \mathbb{E}(\Mb)\|_2 \leq \xi \|\mathbb{E}(\Mb)\|_2$, from Weyl's inequality,
we have
\$
\hat{\lambda}_1 \geq 4(1 - \xi)\bigl[\phi(f) + 1 - \mu_0^2\bigr],\quad \text{and~~}\hat{\lambda}_2 \leq 4\xi\phi(f) + 4(1 + \xi) (1 - \mu_0^2). 
\$
Furthermore, for any $\gamma \in \left({(1 - \mu_0^2)}\big/{\bigl[\phi(f) + 1 - \mu_0^2\bigr]},1\right)$, by restricting 
\#\label{eq:w1241}
\xi \leq \frac{\gamma\phi(f) + (\gamma-1)(1 - \mu_0^2)}{(1+\gamma)\bigl[\phi(f)+1-\mu_0^2\bigr]},
\#
we have
\$
\bigl\|\bbeta^{t} - \hat{\bbeta}\bigr\|_2 \leq \sqrt{\frac{1 - \alpha^2}{\alpha^2}}\cdot \gamma^t.
\$
Now we turn to the statistical error. By Wedin's sin theorem, for some positive constant $C > 0$, we have 
\begin{align}
\label{sine}
\sin\angle \bigl(\bbeta^*, \hat{\bbeta}\bigr) & \leq C\cdot \frac{\xi\|\EE(\Mb)\|_2}{\lambda_1 - \lambda_2}.
\end{align}
Elementary calculation yields 
\begin{equation}
\label{sine1}
\bigl\|\hat{\bbeta} - \bbeta^* \bigr\|_2 = 2 \sin \bigl[\angle \bigl(\bbeta^*, \hat{\bbeta}\bigr) / 2 \bigr] \leq \sqrt{2}\sin\angle \bigl(\bbeta^*, \hat{\bbeta}\bigr).
\end{equation}
As $\xi \lesssim \sqrt{p/n}$, combining (\ref{sine}) and (\ref{sine1}), we have 
\$
\bigl\|\hat{\bbeta} - \bbeta^* \bigr\|_2 \lesssim \frac{\phi(f)+1-\mu_0^2}{\phi(f)}\cdot \sqrt{\frac{p}{n}}.
\$
Putting all pieces together, we conclude that if $\xi$ satisfies \eqref{eq:w1241} and $n \gtrsim {p}/{\xi^2}$, then we have that with probability at least $1 - 2e^{-Cp}$,
\$
\bigl\|\bbeta^{t} - \bbeta^*\bigr\|_2 \leq C \cdot \frac{\phi(f)+1-\mu_0^2}{\phi(f)}\cdot \sqrt{\frac{p}{n}} + \sqrt{\frac{1 - \alpha^2}{\alpha^2}}\cdot \gamma^t.
\$
as required.

\subsection{Proof of Theorem \ref{thm:high}}\label{ap:3}
The analysis of Algorithm \ref{alg:high} follows from a combination of \cite{vu2013fantope} (for the initialization via convex relaxation) and \cite{yuan2011truncated} (for the original truncated power method). Recall that $\kappa$ is defined in \eqref{def:kappa}. Assume the initialization $\bbeta^0$ is $\hat{s}$-sparse with $\| \bbeta^0 \|_2 = 1$, and satisfies 
\begin{align}\label{eq:20}
\bigl\| \bbeta^0 - \bbeta^* \bigr\|_2 \leq C\min \left\{ \sqrt{{\kappa (1-\kappa^{1/2})}/{2}}, {\sqrt{2\kappa}}/{4}\right\}, 
\end{align}
for $\hat{s} = C'\max\left\{\left\lceil 1/{(\kappa^{-1/2} - 1)^2} \right\rceil , 1\right\} \cdot s$. Theorem 1 of \cite{yuan2011truncated} implies that 
\$
\bigl\|\bbeta^{t} - \bbeta^*\bigr\|_2 \leq C'' \cdot \frac{\bigl[\phi(f) + (1 - \mu_0^2)\bigr]^{\frac{5}{2}} (1 - \mu_0^2)^{\frac{1}{2}}}{\phi^3} \cdot \sqrt{\frac{s \log p}{n}} + \kappa^t \cdot \sqrt{\min \Bigl\{   (1-\kappa^{1/2})/2, 1/8 \Bigr\}} 
\$
with high probability. Therefore, we only need to prove the initialization $\bbeta^0$ obtained in Algorithm \ref{alg:high} satisfies the condition in \eqref{eq:20}.

Corollary 3.3 of \cite{vu2013fantope} shows that the minimizer to the minimization problem in line 3 of Algorithm \ref{alg:high} satisfies
\$
\bigl\| \bPi^0 - \bbeta^* \cdot (\bbeta^*)^\top \bigr\|_2 \leq C\cdot \frac{\phi(f) + (1 - \mu_0^2)}{\phi} \cdot s \sqrt{\frac{\log p}{n}}
\$
with high probability. Corollary 3.2 of \cite{vu2013fantope} implies, the first eigenvector of $\bPi^0$, denoted as $\overbar{\bbeta}^0$, satisfies 
\$
\bigl\| \overbar{\bbeta}^0  - \bbeta^* \bigr\|_2 \leq C'\cdot \frac{\phi(f) + (1 - \mu_0^2)}{\phi} \cdot s \sqrt{\frac{\log p}{n}}
\$
with the same probability. However, $\overbar{\bbeta}^0$ is not necessarily $\hat{s}$-sparse. Using Lemma 12 of \cite{yuan2011truncated}, we obtain that the truncate step in lines 12-15 of Algorithm 2 ensures that $\bbeta^0$ is $\hat{s}$-sparse and also satisfies 
\$
\bigl\| \bbeta^0 - \bbeta^* \bigr\|_2 \leq \big(1 + 2\sqrt{\hat{s}/s}\big) \cdot \bigl\| \overbar{\bbeta}^0 - \bbeta^* \bigr\|_2 \leq 3\bigl\| \overbar{\bbeta}^0 - \bbeta^* \bigr\|_2,
\$
where the last inequality follows from our assumption that $\hat{s} \geq s$. Therefore, we only have to set $n$ to be sufficiently large such that 
\$
\bigl\| {\bbeta}^0  - \bbeta^* \bigr\|_2 \leq C'\cdot\frac{\phi(f) + (1 - \mu_0^2)}{\phi(f)} \cdot s \sqrt{\frac{\log p}{n}} \leq C\min \biggl\{ \sqrt{{\kappa (1-\kappa^{1/2})}/{2}}, {\sqrt{2\kappa}}/{4}\biggr\}, 
\$
which is ensured by setting $n \geq n_{\min}$ with 
\$
n_{\min} = C' \cdot s^2 \log p \cdot \phi(f)^2 \cdot \min \bigl\{  \kappa (1-\kappa^{1/2})/2, \kappa/8 \bigr\} \big/ \left[  (1 - \mu_0^2) + \phi(f) \right]^2,
\$
as specified in our assumption. Thus we conclude the proof.

\subsection{Proof of Theorem \ref{thm:minimax_high}}
\label{proof:thm:minimax_high}
The proof of the minimax lower bound follows from the basic idea of reducing an estimation problem to a testing problem, and then invoking Fano's inequality to lower bound the testing error. We first introduce a finite packing set for $\mathbb{S}^{p-1} \cap \mathbb{B}_0(s,p)$.
\begin{lemma} \label{thm:finiteset}
	Consider the set $\{0,1\}^p$ equipped with Hamming distance $\delta$. For $s \leq p/4$, there exists a finite subset $\cQ \subset \{0,1\}^p$ such that 
	\$
	\delta(\btheta,\btheta') > s/2,~~ \forall (\btheta,\btheta') \in \cQ \times \cQ ~~\text{and} ~~ \btheta \ne \btheta', ~~\|\btheta\|_0 = s, ~~\text{for~all~} \btheta \in \cQ. 
	\$
The cardinality of such a set satisfies 
	\$
	\log(|\cQ|) \geq  8/3\cdot s \log (p/s).
	\$
\end{lemma}
\begin{proof}
	See the proof of Lemma 4.10 in \cite{massart2007concentration}.
\end{proof}
We use $\cQ(p,s)$ to denote the finite set specified in Lemma \ref{thm:finiteset}. For $\xi < 1$, we construct a finite subset $\overbar{\cQ}(p,s,\xi) \subset \mathbb{S}^{p-1}\cap \mathbb{B}_0(s,p)$ as
\begin{equation}
\label{packingset}
\overbar{\cQ}(p,s,\xi) := \left\{ \bbeta \in \mathbb{R}^p : \bbeta = \biggl(\sqrt{1 - \xi^2}, \frac{\xi}{\sqrt{s-1}} \cdot \bw \biggr), \quad\text{where~~} \bw \in \cQ(p-1,s-1) \right\}.
\end{equation}
It is easy to verify that set $\overbar{\cQ}(p,s,\xi)$ has the following properties:
\begin{itemize}
	\item For any $\btheta \in \overbar{\cQ}(p,s,\xi)$, it holds that $\|\btheta\|_2 = 1$ and $\|\btheta\|_0 = s$.
	\item For distinct $\btheta, \btheta' \in \overbar{\cQ}(p,s,\xi)$, $\|\btheta - \btheta'\|_2 \geq {\sqrt{2}}\xi/2$ and $\|\btheta - \btheta'\|_2 \leq \sqrt{2}\xi$.
	\item $\log |\overbar{\cQ}(p,s,\xi)| \geq Cs\log(p/s)$ for some positive constant $C$.
\end{itemize}
In order to derive lower bound of $\mathcal{R}(n,m,L,\cB)$ with $\cB = \SSS^{p-1}\cap\B_0(s,p)$, we assume that the infimum over $f$ in \eqref{minimax} is obtained for a certain $f^* \in \cF(m,L)$, namely 
\begin{equation*}
\mathcal{R}(n,m,L, \cB) =  \inf_{\hat{\bbeta} \in \mathbb{S}^{p-1}} \sup_{\bbeta \in \SSS^{p-1}\cap\B_0(s,p)} \mathbb{E} \bigl\|\hat{\bbeta}(\cX_{f^*}^n) - \bbeta \bigr\|_2 \geq \inf_{\hat{\bbeta} \in \mathbb{S}^{p-1}} \sup_{\bbeta \in \overbar{\cQ}(p,s,\xi)} \mathbb{E} \bigl\|\hat{\bbeta}(\cX_{f^*}^n) - \bbeta \bigr\|_2. 
\end{equation*}
Note that for any $\xi > 0$, we have $\|\bbeta_1 - \bbeta_2\|_2 \geq \frac{\sqrt{2}}{2}\xi$ for any two distinct vectors $(\bbeta_1,\bbeta_2)$ in $\overbar{\cQ}(p,s,\xi)$. Therefore, we are in a position to apply standard minimax risk lower bound. Following Lemma 3 in \citet{yu1997assouad}, we obtain
\begin{equation} \label{eq:tmp1}
\inf_{\hat{\bbeta} \in \mathbb{S}^{p-1}} \sup_{\bbeta \in \overbar{\cQ}(p,s,\xi)} \mathbb{E} \bigl\|\hat{\bbeta}(\cX_{f^*}^n) - \bbeta \bigr\|_2 \geq   \frac{\sqrt{2}}{4}\xi \left(1 - \frac{\max_{\bbeta,\bbeta' \in \overbar{\cQ}(p,s,\xi) }D_{KL}(P_{\bbeta'}\|P_{\bbeta}) + \log 2}{\log |\overbar{\cQ}(p,s,\xi)|}\right).
\end{equation}
In the following, we derive an upper bound for the term involving KL divergence on the right hand side of the above inequality. For any $\bbeta, \bbeta' \in \overbar{\cQ}(p,s,\xi)$, we have
\begin{align}
\label{KL1}
& D_{KL}(P_{\bbeta'}\|P_{\bbeta}) \leq  n \cdot D_{KL} \big[P_{\bbeta'}(Y,\bX) \| P_{\bbeta}(Y,\bX) \big] = n \cdot \EE_{\bX} \bigl\{ D_{KL}\big[P_{\bbeta'}(Y|\bX) \| P_{\bbeta}(Y|\bX)\big]  \bigr\} \notag\\
&=  \frac{1}{2} n \cdot \EE_{\bX} \bigg\{ \big[1+f^*( \bX^{\top}\bbeta)\big] \log \frac{1 + f^*( \bX^{\top}\bbeta)}{1 + f^*( \bX^{\top}\bbeta')}  +  \big[1 - f^*( \bX^{\top}\bbeta)\big] \log \frac{1 - f^*( \bX^{\top}\bbeta)}{1 - f^*( \bX^{\top}\bbeta')} \bigg\} \notag\\
&\leq  \frac{1}{2} n \cdot \EE_{\bX} \bigg\{ \big[ 1+f^*( \bX^{\top}\bbeta)\big]\left[\frac{1+f^*( \bX^{\top}\bbeta)}{1+f^*( \bX^{\top}\bbeta')} - 1\right] + \big[1-f^*( \bX^{\top}\bbeta)\big]\left[\frac{1-f^*( \bX^{\top}\bbeta)}{1-f^*( \bX^{\top}\bbeta')} - 1 \right] \bigg\}. 
\end{align}
In the last inequality, we utilize the fact that $\log z \leq z - 1$. Then by elementary calculation, we have
\begin{align}
\label{KL2}
D_{KL}(P_{\bbeta'}\|P_{\bbeta}) \leq n \cdot \EE_{\bX}\left\{\frac{\big[f^*( \bX^{\top}\bbeta) - f^*( \bX^{\top}\bbeta')\big]^2} {\big[1 + f^*( \bX^{\top}\bbeta')\big]\cdot\big[1 - f^*( \bX^{\top}\bbeta')\big]}\right\}. 
\end{align}
Using $|f(z)| \leq 1 - m$ and the Lipschitz continuity condition of $f$, we have
\begin{align}
\label{KL3}
D_{KL}(P_{\bbeta'}\|P_{\bbeta}) \leq &  n\cdot \EE_{\bX} \left\{\frac{L^2\langle \bX, \bbeta - \bbeta'\rangle^2}{m(1-m)} \right\} = \frac{nL^2\|\bbeta - \bbeta'\|_2^2}{m(1-m)} \leq  \frac{2nL^2\xi^2}{m(1-m)}.
\end{align}
Note that \eqref{KL1}-\eqref{KL3} hold for any $\bbeta,\bbeta'  \in \overbar{\cQ}(p,s,\xi)$. We thus have 
\$
\max_{\bbeta,\bbeta' \in \overbar{\cQ}(p,s,\xi) }D_{KL}(P_{\bbeta'}\|P_{\bbeta}) \leq \frac{2nL^2\xi^2}{m(1-m)}.
\$
Now we proceed with \eqref{eq:tmp1} using the above result. The right hand side is thus lower bounded by
\begin{align*} 
\frac{\sqrt{2}}{4}\xi\left(1 - \frac{{2L^2n\xi^2}/{[m(1-m)]} + \log 2 }{|\overbar{\cQ}(p,s,\xi)|}\right) \geq  \frac{\sqrt{2}}{4}\xi\left(1 - \frac{{2L^2n\xi^2}/{[m(1-m)]} + \log 2 }{Cs\log(p/s)}\right),
\end{align*}
where the last inequality is from $|\overbar{\cQ}(p,s,\xi)| \geq Cs\log(p/s)$. Finally, consider the case where the sample size $n$ is sufficiently large such that 
\$
n \geq \frac{m(1-m)}{2L^2}\cdot\bigl[Cs\log(p/s)/2 - \log 2\bigr],
\$
by choosing 
\begin{equation} \label{eq:tmp2}
\xi^2  = \frac{m(1-m)}{2L^2n}\cdot\bigl[Cs\log(p/s)/2 - \log 2\bigr],
\end{equation}
 we thus have
\$
\mathcal{R}(n,m,L, \cB) \geq C'\cdot \frac{\sqrt{m(1-m)}}{L}\cdot \sqrt{\frac{s\log(p/s)}{n}}
\$
as required.